\documentclass[10pt]{article}
\usepackage{times}

\pdfoutput=1

\usepackage{microtype}
\usepackage{graphicx}
\usepackage{subfigure}
\usepackage{booktabs}
\usepackage{multirow}

\usepackage{color}

\usepackage{algorithm}
\usepackage{algorithmic}

\usepackage{tabu}
\usepackage{url}

\usepackage{authblk}

\usepackage{parskip}

\newlength{\defbaselineskip}
\usepackage{etoolbox}
\newtoggle{arxiv}
\toggletrue{arxiv}

\usepackage{cdesa-math}

\usepackage{mathtools}
\usepackage[dvipsnames]{xcolor}

\usepackage[numbers,sort]{natbib}

\usepackage{hyperref}
\usepackage{cleveref}

\newlength{\myfcwidth}
\newlength{\mydiagwidth}
\newcommand{\sysname}{HALP}
\newcommand{\systitle}{High-Accuracy Low-Precision Training}

\begin{document}

\title{\systitle{}}

\author[$\dagger$]{Christopher De Sa}
\author[$\ddagger$]{Megan Leszczynski}
\author[$\ddagger$]{Jian Zhang}
\author[$\dagger$]{Alana Marzoev}
\author[$\ddagger$]{Christopher R. Aberger}
\author[$\ddagger$]{Kunle Olukotun}
\author[$\ddagger$]{Christopher R{\'e}}
\affil[$\dagger$]{Department of Computer Science, Cornell University}
\affil[$\ddagger$]{Department of Computer Science, Stanford University\vspace{4pt}}
\affil[ ]{\footnotesize{\texttt{cdesa@cs.cornell.edu}, \texttt{mleszczy@stanford.edu}, \texttt{zjian@stanford.edu}, \texttt{mam655@cornell.edu}, \texttt{caberger@stanford.edu}, \texttt{kunle@stanford.edu}, \texttt{chrismre@cs.stanford.edu}}}

\maketitle

\newcommand{\numb}[1]{{\color{black}#1}}
\setlength{\textfloatsep}{10pt}
\begin{abstract}
Low-precision computation is often used to lower the time and energy cost of machine learning, and recently hardware accelerators have been developed to support it.
Still, it has been used primarily for inference---not training.
Previous low-precision training algorithms suffered from a fundamental tradeoff: as the number of bits of precision is lowered, quantization noise is added to the model, which limits statistical accuracy.
To address this issue, we describe a simple low-precision stochastic gradient descent variant called \sysname{}.
\sysname{} converges at the same theoretical rate as full-precision algorithms despite the noise introduced by using low precision throughout execution.
The key idea is to use SVRG to reduce gradient variance, and to combine this with a novel technique called \emph{bit centering} to reduce quantization error.
We show that on the CPU, \sysname{} can run up to $\numb{4} \times$ faster than full-precision SVRG and can match its convergence trajectory.
We implemented \sysname{} in TensorQuant, and show that it exceeds the validation performance of plain low-precision SGD on two deep learning tasks.
\end{abstract}

\section{Introduction}
\label{introduction}

Many  machine learning training tasks can be written as an optimization problem over a finite sum of $N$ components
\begin{equation}
  \label{eqnFiniteSumProblem}
  \mbox{minimize } f(w) = \frac{1}{N} \sum_{i=1}^N f_i(w) \hspace{1em}
  \mbox{over } w \in \R^d.
\end{equation}
A standard way of solving these optimization problems over very large training datasets is by using \emph{stochastic gradient descent} (SGD)~\cite{rumelhart1986learning,bottou1991stochastic,bottou2012stochastic}.
Given that training for deep neural networks can take weeks, it is important to produce results quickly and efficiently.
For machine learning \emph{inference} tasks, speed and efficiency has been greatly improved by the use of hardware accelerators such as Google's TPU \cite{jouppi2017datacenter} and Microsoft's Project Brainwave \cite{projectbrainwave,caulfield2017configurable}.
Much of the benefit of these accelerators comes from their use of \emph{low-precision arithmetic}, which reduces the overall cost of computation by reducing the number of bits that need to be processed.
These accelerators have been primarily used for inference and not training, partially because the effects of precision on training are not yet well understood.
This motivates us to study how low-precision can be used to speed up algorithms for solving training problems like (\ref{eqnFiniteSumProblem}).

\begin{table*}[t]
\caption{Asymptotic runtimes of our algorithms on linear models compared with other algorithms, to produce an output within an objective function gap $\epsilon$ from the true solution of a finite-sum strongly convex optimization problem with $N$ components and condition number $\kappa$. In this table, ``FP'' means full-precision, and ``LP'' means low-precision. Note that in every case, $O(\log(1/\epsilon))$ bit full-precision numbers are needed to even represent a solution with objective gap $\epsilon$. To compute the overall runtime, we suppose that the cost of an arithmetic OP is proportional to the number of bits used.}
\label{tabSummary}
\vskip 0.15in
\begin{center}
\small
\begin{small}
\begin{sc}
\begin{tabular}{lccccc}
\toprule
Algorithm & overall runtime & \# of FP ops & \# of LP ops & \# of LP bits  \\
\midrule
SGD & $O( \kappa \log(1/\epsilon) / \epsilon )$ & $O( \kappa / \epsilon )$ & --- & ---  \\
SVRG & $O( (N + \kappa) \log^2(1/\epsilon) )$ & $O( (N + \kappa) \log(1/\epsilon) )$ & --- & --- \\
\textbf{LP-SVRG} & $O( (N + \kappa) \log^2(1/\epsilon) )$ & $O( N \log(1/\epsilon) )$ & $O(\kappa \log(1/\epsilon))$ \\
\textbf{\sysname{}} & $O( N \log^2(1/\epsilon) + \kappa \log(\kappa) \log(1/\epsilon) )$ & $O( N \log(1/\epsilon) )$ & $O(\kappa \log(1/\epsilon))$ & $2 \log(O(\kappa))$ \\
\bottomrule
\end{tabular}
\end{sc}
\end{small}
\end{center}
\vskip -0.1in
\end{table*}

Unfortunately, the systems benefits of low-precision (LP) arithmetic come with a cost.
The round-off or \emph{quantization error} that results from converting numbers into a low-precision representation introduces noise that can affect the convergence rate and accuracy of SGD.
Conventional wisdom says that, for training, low precision introduces a tradeoff of the number-of-bits used versus the statistical accuracy---the fewer bits used, the worse the solution will become. Theoretical upper bounds on the performance of low-precision SGD~\cite{desa2015tamingthewild} and empirical observations of implemented low-precision algorithms~\cite{courbariaux2014training,gupta2015deep,desa2017dmgc,zhang2017zipml} further confirm that current algorithms are limited by this precision-accuracy tradeoff.\footnote{\scriptsize A simple way to avoid this and make an algorithm of arbitrary accuracy would be to increase the number of bits of precision as the algorithm converges.
However, this is unsatisfying as it increases the cost of computation, and we want to be able to run on specialized low-precision accelerators that have a fixed bit width.}

In this paper, we upend this conventional wisdom by showing that \emph{it is still possible to get high-accuracy solutions from low-precision training}, as long as the problem is sufficiently well-conditioned.
We do this with an algorithm called \sysname{} which transcends the accuracy limitations of ordinary low-precision SGD.
We address noise from gradient variance using a known technique called SVRG, stochastic variance-reduced gradient~\cite{johnson2013accelerating}.
To address noise from quantization, we introduce a new technique called \emph{bit centering}.
The intuition behind bit centering is that as we approach the optimum, the gradient gets smaller in magnitude and in some sense carries less information, so we should be able to compress it.
By dynamically re-centering and re-scaling our low-precision numbers, we can lower the quantization noise asymptotically as the algorithm converges.
We prove that, for strongly convex problems, \sysname{} is able to produce arbitrarily accurate solutions with the same linear asymptotic convergence rate as SVRG, while using low-precision iterates with a fixed number of bits.
Our theory also exposes a novel tradeoff between condition number $\kappa$ and precision which suggests that the number of bits needed for linear convergence is $b = \log(O(\kappa))$.
Our contributions are as follows:
\begin{itemize}
  \item In \Cref{sec:lp_svrg}, we introduce and study low-precision SVRG (LP-SVRG), which has no bit centering step. We prove that LP-SVRG converges at the same linear rate as SVRG, but (as conventional wisdom would predict) only converges down to an accuracy limit caused by the low-precision arithmetic. 
  \item In \Cref{sec:halp}, we introduce \sysname{}, \emph{High-Accuracy Low-Precision}, and prove that it converges at the same linear rate as SVRG down to solutions of arbitrarily high accuracy, even though it uses a fixed number of bits of precision for its iterates.
          \item  In \Cref{sec:evaluation}, we show that on a CPU, \sysname{} can compute iterations up to $\numb{3} \times$ faster than plain SVRG on the MNIST dataset and up to $\numb{4} \times$ faster than plain SVRG on a synthetic dataset with 10,000 features. We also evaluate our method as a new algorithm for deep learning. We implement our algorithms in TensorQuant \cite{loroch2017tensorquant} and show that when training deep models\footnote{To simulate training deep models in this paper we ran the computation at full-precision then quantized the updates using the algorithms presented in \Cref{sec:lp_svrg,sec:halp}.} our validation performance can match SVRG and exceed low-precision SGD.
\end{itemize}

Our results about asymptotic convergence rates and time complexity, compared with standard rates for SGD and SVRG, are summarized in Table~\ref{tabSummary}.

\section{Related work}

Motivated by the increasing time and energy cost of training large-scale deep learning models on clusters, several recent projects have investigated decreasing these costs using low-precision arithmetic.
It has been folklore for many years that neural network inference could be done effectively even with 8-bit arithmetic~\cite{vanhoucke2011improving}, and there has been much work recently on \emph{compressing} already-trained networks by (among other things) making some of the weights and activations low-precision~\cite{anwar2015fixed,han2016deep,tan2018transparent}.
This interest in low-precision arithmetic for inference has also led to the development of new hardware accelerators for low-precision machine learning, such as Google's TPU which is based on 8-bit low-precision multiplies~\cite{jouppi2017datacenter}.

Work has also been done on evaluating and guaranteeing the effectiveness of low-precision training. Researchers have gathered empirical evidence for low-precision training in specific settings, although these results have typically not produced empirical support for 8-bit training~\cite{savich2011resource,seide2014onebit,courbariaux2014training,gupta2015deep,strom2015scalable,micikevicius2017mixed}. Researchers have also proven bounds on the error that results from using low-precision computation on convex problems and non-convex matrix recovery problems~\cite{desa2015tamingthewild}.
Recently, \citet{zhang2017zipml} has developed techniques called double sampling and optimal quantization which enable users to quantize the training dataset with provable guarantees on the accuracy for convex linear models.
Using these techniques, they designed and evaluated a hardware accelerator that computes low-precision SGD efficiently.
A similar evaluation of the hardware efficiency of low-precision methods on commodity hardware was done by \citet{desa2017dmgc}, which outlined how quantizing in different ways has different effects on accuracy and throughput when SGD is made low-precision.
While these works showed that low-precision training has many benefits, they all observe that accuracy degrades as precision is decreased.

While SGD is a very popular algorithm, the number of iterations it requires to achieve an objective gap of $\epsilon$ for a strongly convex problem is $O(1 / \epsilon)$.
In comparison, ordinary gradient descent (GD) has an asymptotic rate of $O(\log(1/\epsilon))$, which is known as a \emph{linear rate}\footnote{\scriptsize This is called a linear rate because the number of iterations required is linear in the number of significant figures of output precision needed.}, and is asymptotically much faster than SGD.
There has been work on modifications to SGD that preserve its high computational throughput while also recovering the linear rate of SGD~\cite{roux2012stochastic,shalev2013stochastic}.
SVRG is one such method, which recovers the linear rate of gradient descent for strongly convex optimization, while still using stochastic iterations~\cite{johnson2013accelerating}; recently it has been analyzed in the non-convex case as well and shown to be effective in some settings~\cite{allen2016variance,reddi2016stochastic}.
These \emph{variance-reduced} methods are interesting because they preserve the simple hardware-efficient updates of SGD, while recovering the statistically-efficient linear convergence rate of the more expensive gradient descent algorithm. While we are not the first to present theoretical results combining low-precision with SVRG (\citet{alistarh2016qsgd} previously studied using low-precision numbers for communication among workers in parallel SGD and SVRG), to the best our knowledge we are the first to present empirical results of low-precision SVRG and to propose the additional bit centering technique.

\section{Warmup: Mixing low-precision and SVRG}
\label{sec:lp_svrg}
As a warmup, we derive low-precision SVRG (LP-SVRG), which combines low-precision computation with variance reduction (but without the bit centering step of \sysname{}).
We start with the basic SVRG algorithm for minimizing an objective, Algorithm~\ref{algSVRG}.
Compared with standard stochastic gradient descent, SVRG is able to converge at a linear rate because it periodically uses full gradients $\tilde g_k$ to reduce the variance of its stochastic gradient estimators.
Note that the two outer-loop update options come from the paper that originally proposed SVRG, \citet{johnson2013accelerating}.
In this and subsequent work~\cite{harikandeh2015stopwasting}, it has been standard to use option II for the theoretical analysis (as it simplifies the derivation) while using option I for all empirical experiments.
We will continue to do this for all the SVRG variants we introduce here.
To construct LP-SVRG, we will make the SVRG algorithm low-precision by making the model vectors $w$ and $\tilde w$ low-precision.
First, we will give some details about what we mean by a low-precision number and describe how these numbers are quantized.
Second, we will state and explain LP-SVRG.
Third, we will validate LP-SVRG both theoretically and experimentally.

\begin{algorithm}[t]
  \caption{SVRG: Stochastic Variance-Reduced Gradient}
  \begin{algorithmic}
  \small
  \label{algSVRG}
    \STATE \textbf{given:} $N$ loss gradients $\nabla f_i$, number of epochs $K$, epoch length $T$, step size $\alpha$, and initial iterate $\tilde w_1$.
    \FOR{$k = 1$ \textbf{to} $K$}
      \STATE $\tilde g_k \leftarrow \nabla f(\tilde w_k) = \frac{1}{N} \sum_{i=1}^N \nabla f_i(\tilde w_k)$
      \STATE $w_{k,0} \leftarrow \tilde w_k$
      \FOR{$t = 1$ \textbf{to} $T$}
        \STATE \textbf{sample} $i$ uniformly from $\{1, \ldots, N\}$
        \STATE $w_{k,t} \leftarrow w_{k,t-1} - \alpha \left(\nabla f_i(w_{k,t-1}) - \nabla f_i(\tilde w_k) + \tilde g_k \right)$
      \ENDFOR
      \STATE \textbf{option I:} set $\tilde w_{k+1} \leftarrow w_{k,T}$
      \STATE \textbf{option II: sample} $t$ uniformly from $\{0, \ldots, T-1\}$, then set $\tilde w_{k+1} \leftarrow w_{k,t}$
    \ENDFOR
    \STATE \textbf{return} $\tilde w_{K+1}$
  \end{algorithmic}
\end{algorithm}

\textbf{Representation. } 
We will consider low-precision representations which store numbers using a limited number of bits, backed by integer arithmetic.\footnote{\scriptsize Although work has also been done on low-precision floating-point numbers, for simplicity we will not discuss them here.}
Specifically, a low-precision representation is a tuple $(\delta, b)$ consisting of a \emph{scale factor} $\delta \in \R$ and a \emph{number of bits} $b \in \N$.
The numbers representable in this format are
\[
  \text{dom}(\delta, b) = \left\{
    -\delta \cdot 2^{b - 1}, \ldots, -\delta, 0, \delta, \ldots, \delta \cdot (2^{b - 1} - 1)
  \right\}.
\]
This is a generalization of standard fixed-point arithmetic, where the scale factor is allowed to be arbitrary rather than being restricted to powers of two.
Low-precision numbers with the same scale factor can be easily added using integer addition, producing a new number with the same scale factor and the same number of bits.\footnote{\scriptsize The result of an addition will have the same number of bits if saturating addition is used. If exact addition is desired, the number of bits must be increased somewhat to prevent overflow.}
Any two low-precision numbers can be multiplied using an integer multiply, producing a new number with a scale factor that is the product of the input scale factors, and a number of bits that is the sum of the two input bit-counts.
If we restrict ourselves in constructing an algorithm to use mostly additions and multiplies of these forms, we can do most of our computation with efficient low-precision integer arithmetic.

\textbf{Quantization. } Now that we have described low-precision representations, we need some way to convert numbers to store them in these representations.
For reasons that have been explored in other work~\cite{gupta2015deep,desa2015tamingthewild}, for our algorithms here we will use \emph{unbiased rounding} (also known as randomized rounding or stochastic rounding).
This involves using a quantization function $Q$ that chooses to round up or down at random such that for any $x$ that is in the interior of the domain of the low-precision representation, $\Exv{Q(x)} = x$.
If $x$ is not in the interior of the domain, $Q$ outputs the closest representable value to $x$ (which will always be either the largest or smallest representable value).
Here, we let $Q_{(\delta,b)}$ denote the function that quantizes into the low-precision representation $(\delta,b)$.
When we use $Q$ to quantize a vector, we mean that all the components are quantized independently.

\begin{algorithm}[t]
  \caption{LP-SVRG: Low-Precision SVRG}
  \begin{algorithmic}
    \small
    \label{algLPSVRG}
    \STATE \textbf{given:} $N$ loss gradients $\nabla f_i$, number of epochs $K$, epoch length $T$, step size $\alpha$, and initial iterate $\tilde w_1$.
    \STATE \textbf{given:} low-precision representation $(\delta, b)$
    \FOR{$k = 1$ \textbf{to} $K$}
      \STATE $\tilde g_k \leftarrow \nabla f(\tilde w_k) = \frac{1}{N} \sum_{i=1}^N \nabla f_i(\tilde w_k)$
      \STATE $w_{k,0} \leftarrow \tilde w_k$
      \FOR{$t = 1$ \textbf{to} $T$}
        \STATE \textbf{sample} $i$ uniformly from $\{1, \ldots, N\}$
        \STATE $u_{k,t} \leftarrow w_{k,t-1} - \alpha \left(\nabla f_i(w_{k,t-1}) - \nabla f_i(\tilde w_k) + \tilde g_k \right)$
        \STATE \textbf{quantize} $w_{k,t} \leftarrow Q_{(\delta,b)}\left( u_{k,t} \right)$
      \ENDFOR
      \STATE \textbf{option I:} set $\tilde w_{k+1} \leftarrow w_{k,T}$
      \STATE \textbf{option II: sample} $t$ uniformly from $\{0, \ldots, T-1\}$, then set $\tilde w_{k+1} \leftarrow w_{k,t}$
    \ENDFOR
    \STATE \textbf{return} $\tilde w_{K+1}$
  \end{algorithmic}
\end{algorithm}

\textbf{Algorithm. } We now construct Algorithm~\ref{algLPSVRG}, LP-SVRG, which modifies SVRG by storing the model vectors $w$ and $\tilde w$ in a low-precision representation $(\delta, b)$ that is passed as an input to the algorithm.
LP-SVRG accomplishes this by first quantizing any value that SVRG would store to $w$.
As a result, the returned solution is \emph{also} low-precision, which can be useful for many applications, such as deep learning training where we want a low-precision model for later fast inference.
Another potential benefit is that all the example gradient computations $\nabla f_i$ are called on low-precision arguments: for many applications, this can significantly decrease the cost of computing those gradients.

\textbf{Theory. } The natural next question is: how does using low-precision computation affect the convergence of the algorithm?
One thing we can say immediately is that LP-SVRG will not converge asymptotically at a linear rate, as it will be limited to producing outputs in the low-precision representation $(\delta, b)$: once it gets as close as possible to the solution in this representation, it can get no closer, and convergence will stop.
The next-best thing we can hope for is that LP-SVRG will converge at a linear rate until it reaches this limit, at which point it will stop converging---and in fact this is what happens.
But before we can prove this, we need to state some assumptions.
First, we require that the objective $f$ is $\mu$-strongly convex
\[
  (x - y)^T (\nabla f(x) - \nabla f(y)) \ge \mu \norm{x - y}^2
\]
and the gradients $\nabla f_i$ are all $L$-Lipschitz continuous
\[
  \norm{ \nabla f_i(x) - \nabla f_i(y) } \le L \norm{x - y}.
\]
In terms of these parameters, the \emph{condition number} of the problem is defined as $\kappa = L / \mu$. 
These assumptions are standard and are the same ones used for the analysis of SVRG \cite{johnson2013accelerating}.
We also need to assume that the global solution $w^*$ to our problem is within the range of numbers that are representable in our low-precision representation.
To ensure this, we require that for any $j$,
\begin{equation}
  \label{eqn:representationbound}
  -\delta \cdot 2^{b - 1} \le (w^*)_j \le \delta \cdot (2^{b - 1} - 1).
\end{equation}
This is easy to satisfy in practice if we have some bound on the magnitude of $w^*$.
Under these conditions we can provide convergence guarantees for LP-SVRG.
\begin{theorem}
  \label{thmLPSVRG}
  Suppose that we run LP-SVRG (Algorithm~\ref{algLPSVRG}) under the above conditions, using option II for the epoch update.
  For any constant $0 < \gamma < 1$ (a parameter which controls how often we take full gradients),
  if we set our step size and epoch lengths to be
  \[
    \alpha
    =
    \frac{\gamma}{4 L (1 + \gamma)}
    \hspace{2em}
    T
    \ge
    \frac{8 \kappa (1 + \gamma)}{\gamma^2}
  \]
  then the outer iterates of LP-SVRG will converge to an accuracy limit at a linear rate
  \[
    \Exv{ f(\tilde w_{K+1}) - f(w^*) }
    \le
    \gamma^K \left( f(\tilde w_1) - f(w^*) \right)
    +
    \frac{2 d \delta^2 L}{ \gamma (1 - \gamma) }.
  \]
\end{theorem}

As a consequence, the number of outer iterations we need to converge to a distance $\epsilon$ from the limit will be $K = \log( (f(\tilde w_1) - f(w^*)) / \epsilon )$, \emph{which is a linear rate, just like SVRG---but only down to an error that is limited by the precision used}.
Note that this theorem exhibits the same tradeoff between bits-of-precision and accuracy that had been previously observed: as the number of bits becomes smaller, a larger $\delta$ will be needed to satisfy (\ref{eqn:representationbound}), and so the accuracy limit (which is $O(\delta^2)$) will become worse.

\textbf{Validation. } To validate LP-SVRG empirically, we ran it on a synthetic linear regression problem.
Figure~\ref{fig:linear_regress} shows that LP-SVRG, with both 8-bit and 16-bit precision, tracks the convergence trajectory of full-precision SVRG until reaching an accuracy floor that is determined by the precision level.
This behavior matches our theoretical predictions.
The results also show that LP-SVRG matches or outperforms low-precision SGD (LP-SGD) both in terms of convergence speed and the eventual accuracy limit.

\begin{figure}
  \centering
  \includegraphics[width=0.7\myfcwidth]{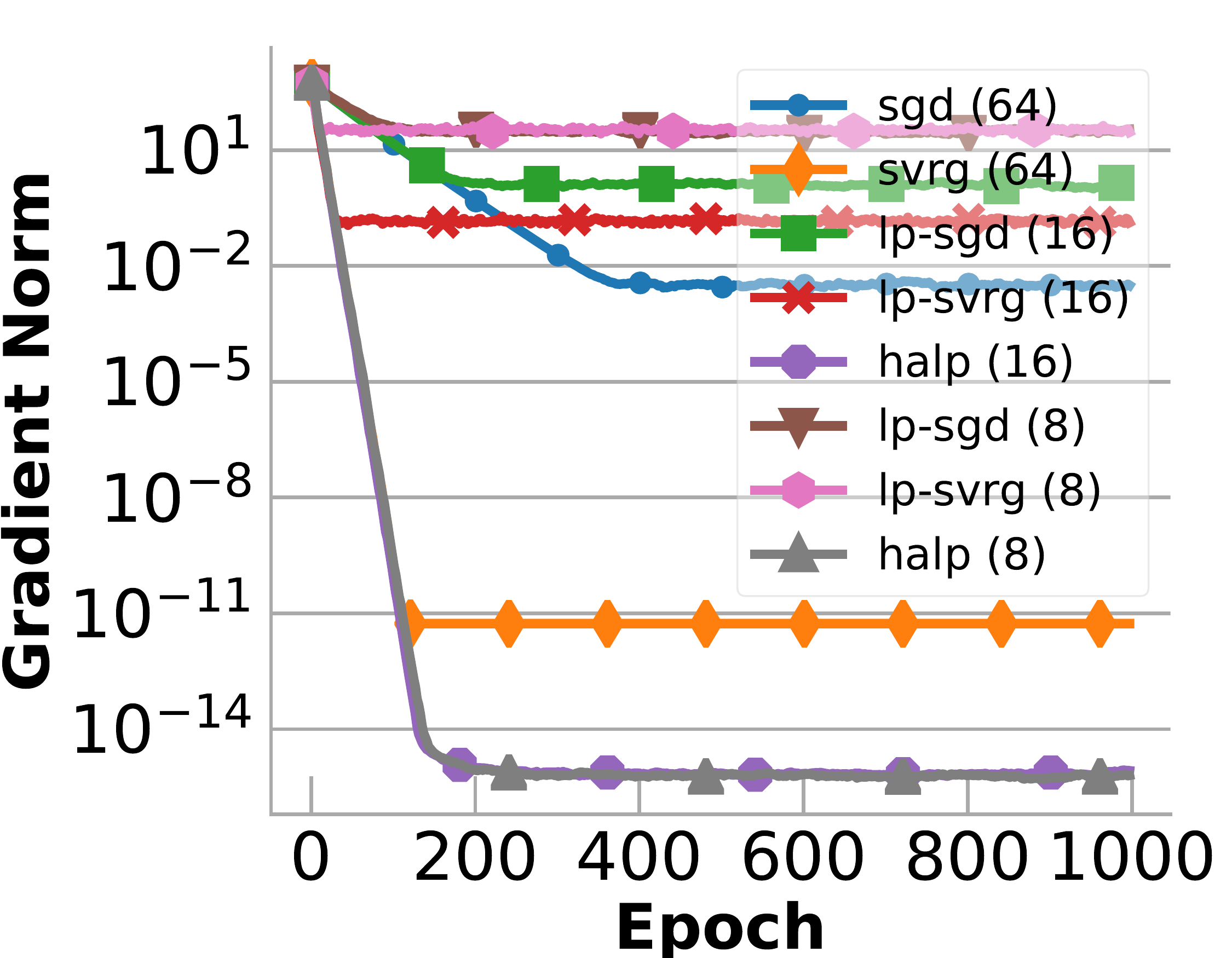}
  \caption{Linear regression on a synthetic dataset with 100 features and 1000 examples generated by scikit-learn's \texttt{make_regression} generator~\cite{scikit-learn}. The epoch length was set to $T = 2000$, twice the number of examples, and the learning rates $\alpha$ and scale factors $\delta$ were chosen using grid search for all algorithms. For all versions of SGD, $\alpha = 2.5 \times 10^{-6}$, and for all versions of SVRG, $\alpha = 5 \times 10^{-3}$. All LP 8-bit algorithms use $\delta = 0.7$ and all LP 16-bit algorithms use $\delta = 0.003$.
  All HALP algorithms use $\alpha = 5 \times 10^{-3}$ and $\mu = 3$.}
  \label{fig:linear_regress}
\end{figure}

\begin{figure*}[t]
\label{figCenterScale}
\centering
\includegraphics[width=\mydiagwidth]{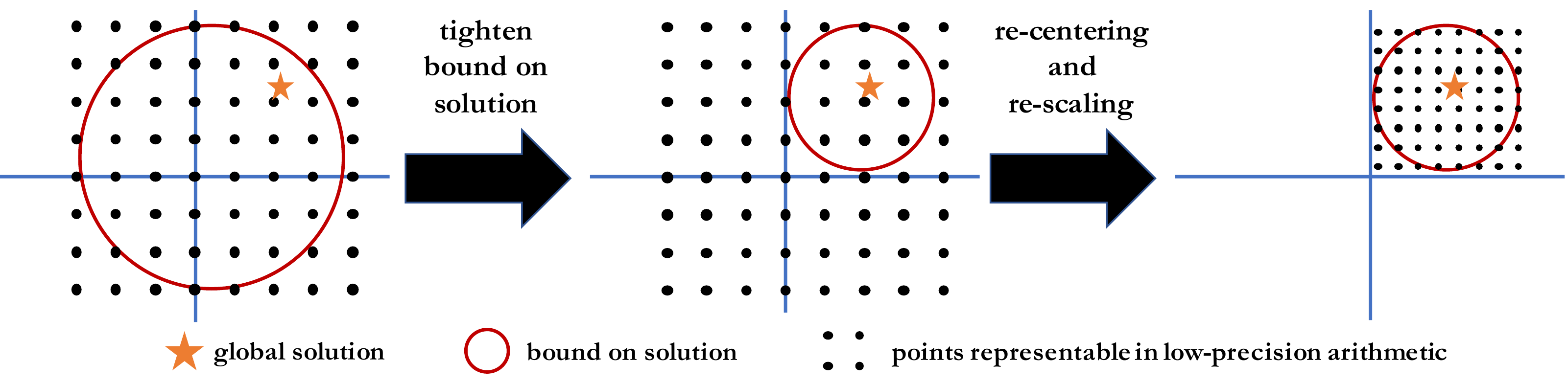}
\vspace{-3mm}

\caption{A diagram of the bit scaling operation in \sysname{}. As the algorithm converges, we are able to bound the solution within a smaller and smaller ball. Periodically, we re-center the points that our low-precision model can represent so they are centered on this ball, and we re-scale the points so that more of them are inside the ball. This decreases quantization error as we converge.}
\end{figure*}

\section{HALP: High-accuracy with low-precision}
\label{sec:halp}
While LP-SVRG converged at a linear rate, it only converged down to a level of accuracy proportional to the delta of quantization of the low-precision representation chosen.
In fact, this is a fundamental limitation of algorithms like LP-SVRG and LP-SGD: we cannot produce a solution that is asymptotically more accurate than the \emph{most accurate solution representable} in the low-precision representation we have chosen.
Since the low-precision representation $(\delta, b)$ in the previous section is chosen \emph{a priori} and is fixed throughout the algorithm, this accuracy limitation is impossible to overcome. While the use of SVRG allowed us to reach this minimum level of accuracy more quickly (at a linear rate, in fact) compared with low-precision SGD~\cite{desa2015tamingthewild}, it has not let us surpass this minimum.

In this section we develop an algorithm that \emph{can} surpass this minimum level of accuracy, and converge to arbitrarily accurate solutions while still using low-precision arithmetic.
First, we will introduce a technique called \emph{bit centering}, which reduces the noise from quantization as the algorithm converges.
Second, we will state and explain our algorithm, \sysname{}.
Third, we will validate \sysname{} by showing theoretically and empirically that it can converge at a linear rate, just like full-precision SVRG.
Finally, we will give some implementation details that show how \sysname{} can be computed efficiently on a class of problems.

\textbf{Bit centering. }
In standard SVRG, each outer iteration is conceptually rewriting the original objective (\ref{eqnFiniteSumProblem}) as
\[
  \frac{1}{N} \sum_{i=1}^N \left(f_i(w) - (w - \tilde w)^T \nabla f_i(\tilde w) + (w - \tilde w)^T \nabla f(\tilde w) \right)
\]
and then running SGD on the rewritten objective.
For our algorithm, \sysname{}, we will do the additional substitution $w = \tilde w + z$, and then minimize over $z$ the objective
\[
  f(\tilde w + z)
  =
  \frac{1}{N} \sum_{i=1}^N \left(f_i(\tilde w + z) - z^T \nabla f_i(\tilde w) + z^T \nabla f(\tilde w) \right).
\]
The reason for this substitution is that as $\tilde w$ comes closer to the solution $w^*$, the range of $z$ we will need to optimize over becomes smaller: by the strong convexity assumption,
\[
  \textstyle
  \norm{ z^* } = \norm{ \tilde w - w^* } \le \frac{1}{\mu} \norm{ \nabla f(\tilde w) }.
\]
This means that if at each outer iteration we \emph{dynamically} reset the low-precision representation to
\[
  (\delta, b) = \left( \frac{ \norm{ \nabla f(\tilde w) } }{\mu (2^{b - 1} - 1)}, b \right)
\]
then we will be \emph{guaranteed} that $z^* = w^* - \tilde w$ will be in the range of points representable in $(\delta, b)$.
Effectively, what we are doing is re-centering and re-scaling the lattice of representable points so that it aligns with our beliefs about where the solution is: this is illustrated in Figure~\ref{figCenterScale}.
Equivalently, we can think about this as being about the gradients: as $\nabla f(\tilde w)$ becomes smaller in magnitude, we can represent it with lower-magnitude error even with a fixed number of bits---but only if we re-scale our low-precision representation.
The important consequence of bit centering is that as the algorithm converges, $\nabla f(\tilde w)$ will become smaller, which will make $\delta$ smaller, which will reduce the quantization error caused by the low-precision arithmetic.

\begin{algorithm}[t]
  \caption{HALP: High-Accuracy Low-Precision SGD}
  \begin{algorithmic}
  \small
  \label{algHALP}
    \STATE \textbf{given:} $N$ loss gradients $\nabla f_i$, number of epochs $K$, epoch length $T$, step size $\alpha$, and initial iterate $\tilde w_1$.
    \STATE \textbf{given:} number of low-precision-representation bits $b$.
    \FOR{$k = 1$ \textbf{to} $K$}
      \STATE $\tilde g_k \leftarrow \nabla f(\tilde w_k) = \frac{1}{N} \sum_{i=1}^N \nabla f_i(\tilde w_k)$
      \STATE $\tilde s_k \leftarrow \frac{ \norm{\tilde g_k} }{\mu (2^{b - 1} - 1)}$
      \STATE \textbf{re-scale: } $\left( \delta, b \right) \leftarrow \left( \tilde s_k, b \right)$
      \STATE $z_{k,0} \leftarrow Q_{(\delta,b)}(0)$
      \FOR{$t = 1$ \textbf{to} $T$}
        \STATE \textbf{sample} $i$ uniformly from $\{1, \ldots, N\}$
        \STATE $u_{k,t} \leftarrow z_{k,t-1} - \alpha \big(\nabla f_i(\tilde w_k + z_{k,t-1})$
        \STATE \hspace{11em}$- \nabla f_i(\tilde w_k) + \tilde g_k \big)$
        \STATE \textbf{quantize: } $z_{k,t} \leftarrow Q_{(\delta,b)}\left( u_{k,t} \right)$
      \ENDFOR
      \STATE \textbf{option I:} set $\tilde w_{k+1} \leftarrow \tilde w_k + z_{k,T}$
      \STATE \textbf{option II: sample} $t$ uniformly from $\{0, \ldots, T-1\}$, then set  $\tilde w_{k+1} \leftarrow \tilde w_k + z_{k,t}$
    \ENDFOR
    \STATE \textbf{return} $\tilde w_{K+1}$
  \end{algorithmic}
\end{algorithm}

\textbf{Algorithm. }
Modifying LP-SVRG by applying bit centering every epoch results in Algorithm~\ref{algHALP}, \sysname{}.
Compared to SVRG, we are now storing $z = w - \tilde w$ in low-precision instead of storing $w$, but otherwise the algorithm has the same structure.
As a result, for problems for which gradients of the form $\nabla f_i(\tilde w + z)$ can be computed efficiently when $z$ is low-precision, \sysname{} can have a higher throughput than SVRG.
Importantly, since $\tilde w$ is still stored in full-precision, the range of $w$ that are representable by the algorithm is still not limited by precision, which is what allows \sysname{} to get arbitrarily close to the optimum.

\textbf{Theory. }
Using the same conditions that we used for LP-SVRG, we can prove that \sysname{} converges at a linear rate, to solutions of \emph{arbitrarily low error}.

\begin{theorem}
  \label{thmHALP}
  Suppose that we run \sysname{} (Algorithm~\ref{algHALP}) under the standard conditions of strong convexity and Lipschitz continuity, using option II for the epoch update.
  For any constant $0 < \gamma < 1$,
  if we use a number of bits
  \[
    \textstyle
    b > 1 + \log_2 \left( 1 + \sqrt{ \frac{2 \kappa^2 d (1 + \gamma)}{\gamma^2} } \right)
  \]
  and we set our step size and epoch lengths to be
  \[
    \alpha = \frac{\gamma}{4 L (1 + \gamma)},
    \,
    T
    \ge
    \frac{
      8 \kappa (1 + \gamma)
    }{
      \gamma^2
      -
      2 \kappa^2 d (1 + \gamma) (2^{b-1} - 1)^{-2}
    },
  \]
  then after running \sysname{} the output will satisfy
  \[
    \Exv{ f(\tilde w_{K+1}) - f(w^*) } \le \gamma^K \left( f(\tilde w_1) - f(w^*) \right).
  \]
\end{theorem}

This theorem shows that we can achieve a linear asymptotic convergence rate even with constant-bit-width low-precision computation in the inner loop.
It also describes an interesting tradeoff between the precision and the condition number.
As the condition number becomes larger while the precision stays fixed, we need to use longer and longer epochs ($T$ becomes larger), until eventually the algorithm might stop working altogether.
This suggests that low-precision training should be combined with techniques to improve the condition number, such as preconditioning.

\textbf{Validation. }
Just as we did for LP-SVRG, we validate these results empirically by running \sysname{} on the same synthetic linear regression problem.
Figure~\ref{fig:linear_regress} shows that \sysname{}, with both 8-bit and 16-bit precision, tracks the linear convergence trajectory of full-precision SVRG until their accuracy becomes limited by the error of the high-precision floating point numbers.
Interestingly, even though all the algorithms in Figure~\ref{fig:linear_regress} use 64-bit floating point numbers for their full-precision computation, \sysname{} actually converges to a solution that is \emph{more accurate}: this happens because when we are very close to the optimum, the quantization error of the floating-point numbers in SVRG actually exceeds that of the low-precision numbers used in \sysname{}.

\textbf{Efficient implementation. }
There are two main challenges to making a \sysname{} implementation fast.
We need to be able to take advantage of low-precision computation when: (1) computing the gradients $\nabla f_i(\cdot)$, and (2) summing up the vectors to compute $u_{k,t}$.
To illustrate how these challenges can be addressed, we choose one common class of objectives, linear models, and show how HALP can be implemented efficiently for these objectives.
(Note that HALP is not limited to simple linear models---we will show empirically that HALP can run in TensorQuant on deep learning tasks.)
A linear model is one in which the objective loss is of the form
$f_i(w) = l_i(x_i^T w)$
for some scalar functions $l_i: \R \rightarrow \R$, which typically depend on the class label, and some \emph{training examples} $x_i$ which we will assume are also stored in a low-precision representation $(\delta_d, b)$.
(For simplicity, we assume that the training examples $x_i$ have the same number of bits as the model variables.)
In this setting, the gradients are
$\nabla f_i(w) = l_i'(x_i^T w) \, x_i$,
and the update step of \sysname{} (before quantization) will be
\begin{align*}
  u_{t}
  &=
  z_{t-1} - \alpha \left(
    \nabla f_i(\tilde w + z_{t-1})
    -
    \nabla f_i(\tilde w)
    +
    \tilde g
  \right) \\
  &=
  z_{t-1} - \alpha \left(
    l_i'\left(x_i^T \tilde w + x_i^T z_{t-1}\right)
    -
    l_i'\left(x_i^T \tilde w \right)
  \right) x_i
  -
  \alpha \tilde g.
\end{align*}
From here, we can address our two challenges individually.
Computing the gradient fast here reduces to computing the dot products $x_i^T z_{t-1}$ and $x_i^T \tilde w$, since the evaluation of $l_i'$ is only a scalar computation.
For the dot product $x_i^T z_{t-1}$, it is easy to take advantage of low-precision arithmetic: this is a dot product of two low-precision vectors, and so can be computed fast using low-precision integer arithmetic.
The dot product $x_i^T \tilde w$ is a bit more tricky: this is unavoidably full-precision (since it involves the full-precision vector $\tilde w$), but it does not depend on $z$, so we can \emph{lift it out} of the inner loop and compute it only once for each training example $x_i$ as we are computing the full gradient in the outer loop.
Computing these two dot products in this way eliminates all full-precision vector operators from the computation of the gradients, which addresses our first challenge.

The second challenge, summing and scaling vectors to compute $u_t$, requires a bit more care.
This task reduces to
\[
  u_{t} = z_{t-1} - \beta x_i - \alpha \tilde g
\]
where $\beta$ is a full-precision scalar that results from the gradient computation.
To make this computation low-precision, we start by using the same lifting trick on $\alpha \tilde g$: we can compute it and quantize it once in the outer loop.
Next the vector-by-scalar product $\beta x_i$ can be approximated efficiently by first \emph{quantizing the scalar} $\beta$, and only then doing the multiply, again using integer arithmetic.
Finally, the summing up of terms to produce $u_t$ can be done easily if we set our scale factors carefully so they are all compatible: in this case, the sum can be done directly with integer arithmetic.
This addresses our second challenge, and eliminates all full-precision vector operations from the inner loop.

\begin{algorithm}[t]
  \caption{LM-\sysname{}: \sysname{} for Linear Models}
  \begin{algorithmic}
  \small
  \label{algLMHALP}
    \STATE \textbf{given:} $N$ loss functions $l_i$ and training examples $x_i$, number of epochs $K$, epoch length $T$, step size $\alpha$, and initial iterate $\tilde w_1$.
    \STATE \textbf{given:} number of low-precision model bits $b$
    \STATE \textbf{given:} low-precision data representation $(\delta_d, b)$
    \FOR{$k = 1$ \textbf{to} $K$}
      \FOR{$i = 1$ \textbf{to} $N$}
        \STATE $\phi_{k,i} \leftarrow x_i^T \tilde w_k$
      \ENDFOR
      \STATE $\tilde g_k \leftarrow \nabla f(\tilde w_k) = \frac{1}{N} \sum_{i=1}^N l'_i(\phi_{k,i}) x_i$
      \STATE $\tilde s_k \leftarrow \frac{ \norm{\tilde g_k} }{\mu (2^{b - 1} - 1)}$
      \STATE \textbf{re-scale:} $\left( \delta_m, b \right) \leftarrow \left( \tilde s_k, b \right)$
      \STATE \textbf{re-scale:} $\left( \delta_i, 2b \right) \leftarrow \left( 2^{-b} \cdot \delta_m, 2b \right)$
      \STATE \textbf{re-scale:} $\left( \delta_s, b \right) \leftarrow \left( 2^{-b} \cdot \delta_m / \delta_d, b \right)$
      \STATE \textbf{quantize:} $\tilde h_k \leftarrow Q_{(\delta_i, b)}(\alpha \tilde g_k)$
      \STATE $z_{k,0} \leftarrow Q_{(\delta_m, b)}(0)$
      \FOR{$t = 1$ \textbf{to} $T$}
        \STATE \textbf{sample} $i$ uniformly from $\{1, \ldots, N\}$
        \STATE $\beta_{k,t} \leftarrow \alpha \left(l'_i(\phi_{k,i} + x_i^T z_{k,t-1}) - l'_i(\phi_{k,i}) \right)$
        \STATE \textbf{quantize:} $\gamma_{k,t} \leftarrow Q_{(\delta_s, b)}(\beta_{k,t})$
        \STATE $u_{k,t} \leftarrow z_{k,t-1} - \gamma_{k,t} x_i - \tilde h_k$
        \STATE \textbf{quantize:} $z_{k,t} \leftarrow Q_{(\delta_m, b)}(u_{k,t})$
      \ENDFOR
      \STATE $\tilde w_{k+1} \leftarrow \tilde w_k + z_{k,T}$
    \ENDFOR
    \STATE \textbf{return} $\tilde w_{K+1}$
  \end{algorithmic}
\end{algorithm}

Explicitly applying our techniques to \sysname{} produces Algorithm~\ref{algLMHALP}.
In Algorithm~\ref{algLMHALP} the model variable $z_{k,t}$ is stored in representation $(\delta_m, b)$, the scale factor $\gamma_{k,t}$ is stored in representation $(\delta_s, b)$, and the offset $\tilde h_k$ is stored in representation $(\delta_i, 2b)$.
Importantly, we can also store the temporary value $u_{k,t}$ in representation $(\delta_i, 2b)$, because: (1) $z_{k,t-1}$ be easily be converted to that representation using a shift-left by $b$ bits, (2) $\gamma_{k,t} x_i$ is already naturally in that representation because $\delta_i = \delta_s \cdot \delta_d$, and (3) $\tilde h_k$ is already in that representation.
As a result, even the quantization to compute $z_{k,t}$ can be done in low-precision, which validates that we have \emph{removed all full-precision vector computations from the inner loop}. We can also use a similar technique to remove all full-precision vector computations from the inner loop of LP-SVRG, although due to space limitations we will not present this algorithm explicitly here.
Counting the number of vector operations used in these algorithms results in the computational complexity numbers presented in Table~\ref{tabSummary}.

\begin{figure*}
\centering
\begin{tabular}{c@{\hspace{0.3em}} c@{\hspace{0.3em}} c@{\hspace{0.3em}} c}
  \subfigure[Training loss on ResNet.]{\label{fig:train_loss_cnn}\includegraphics[width=0.24\linewidth]{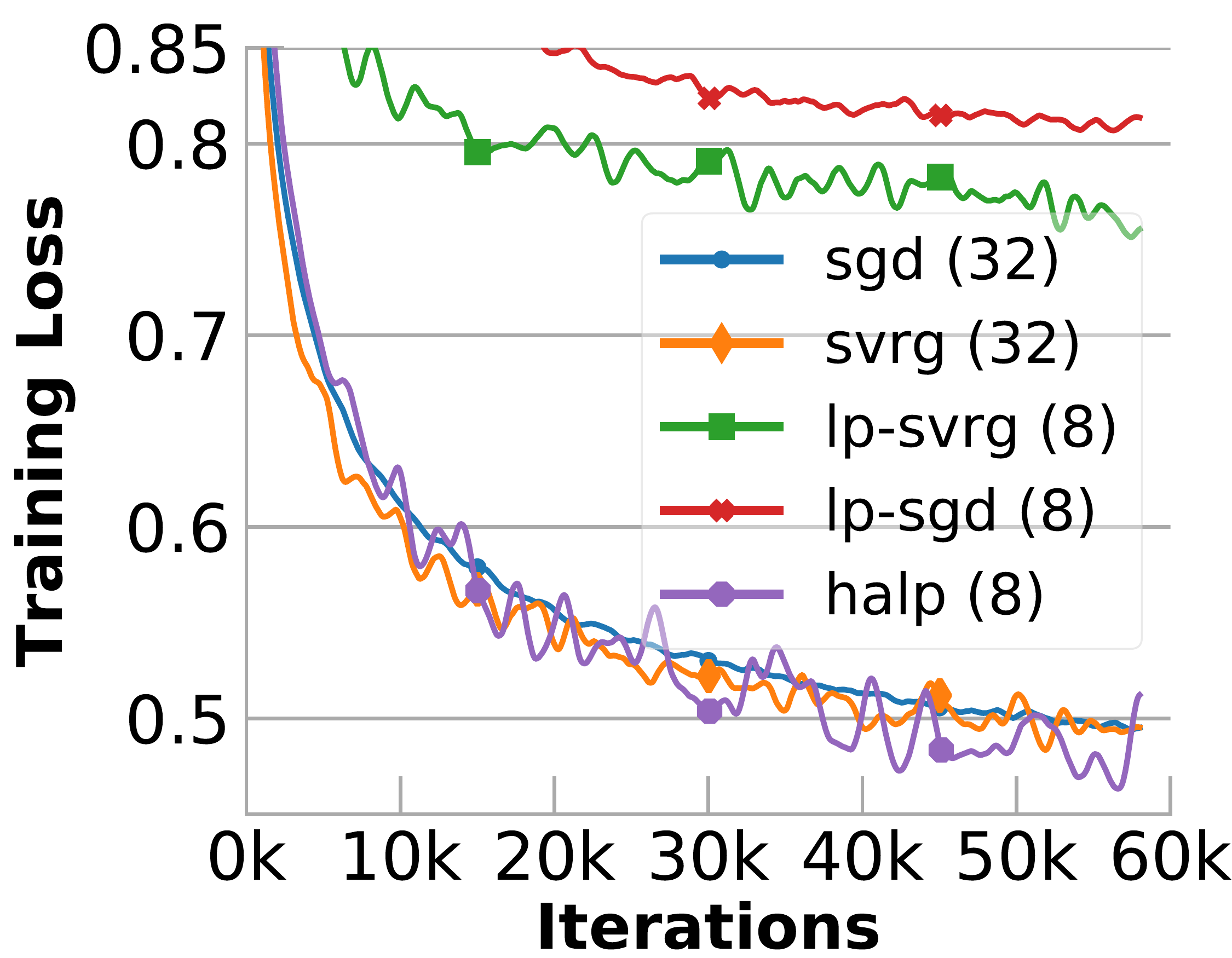}} \hfill
  \subfigure[Val. Accuracy on ResNet.]{\label{fig:val_acc_cnn}\includegraphics[width=0.24\linewidth]{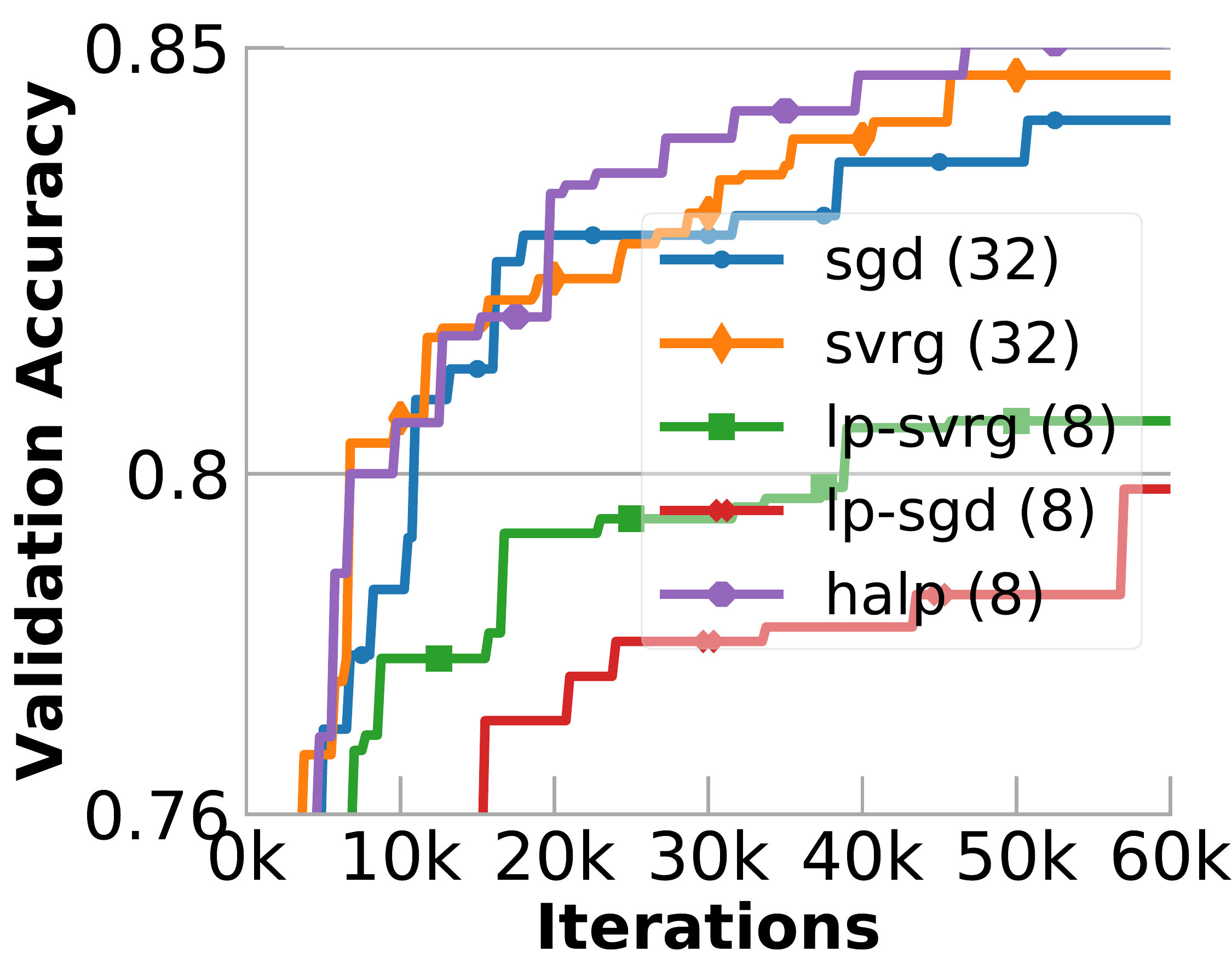}} \hfill
  \subfigure[Training loss on LSTM.]{\label{fig:train_loss_lstm}\includegraphics[width=0.24\linewidth]{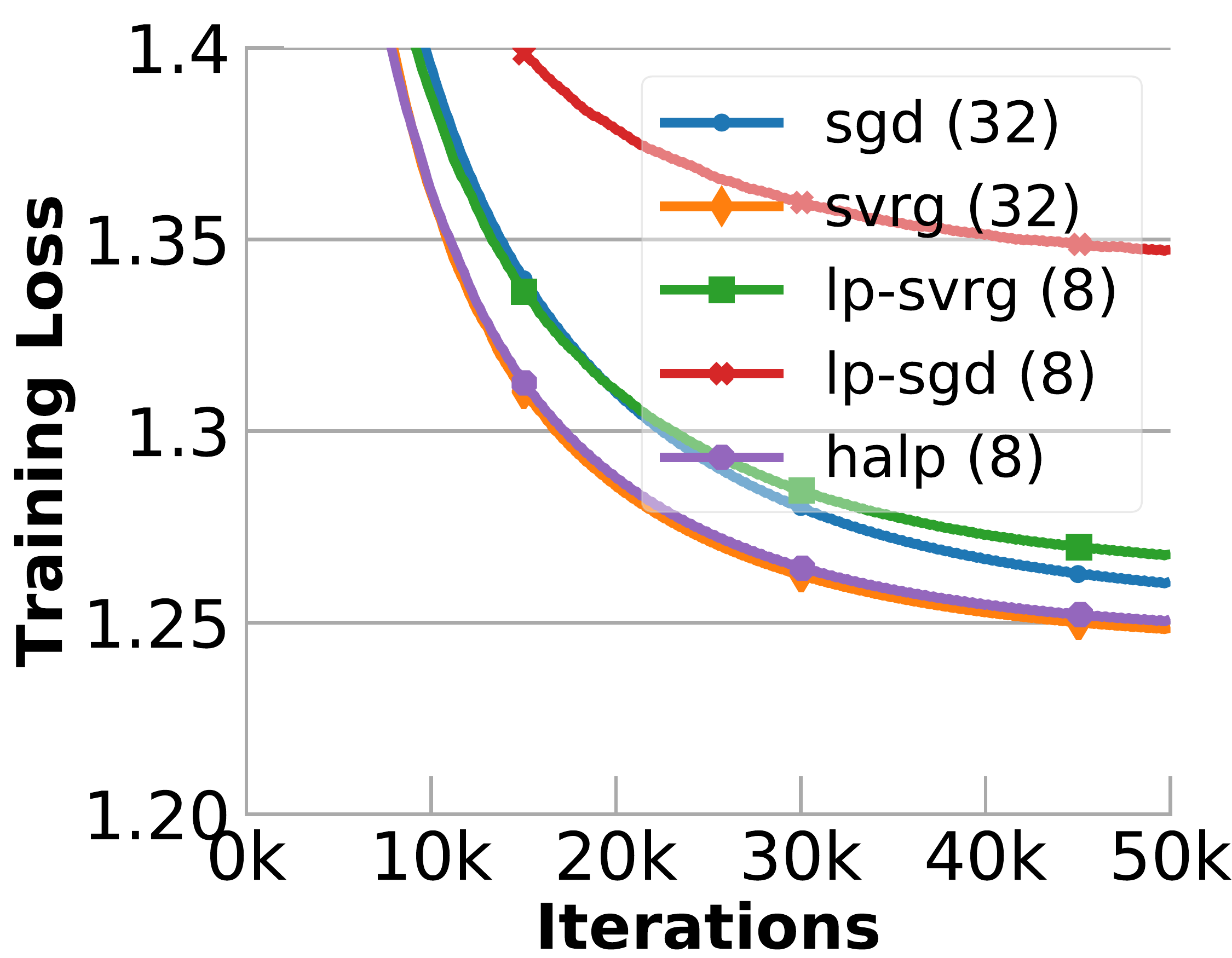}} \hfill
  \subfigure[Val. Accuracy on LSTM.]{\label{fig:val_acc_lstm}\includegraphics[width=0.24\linewidth]{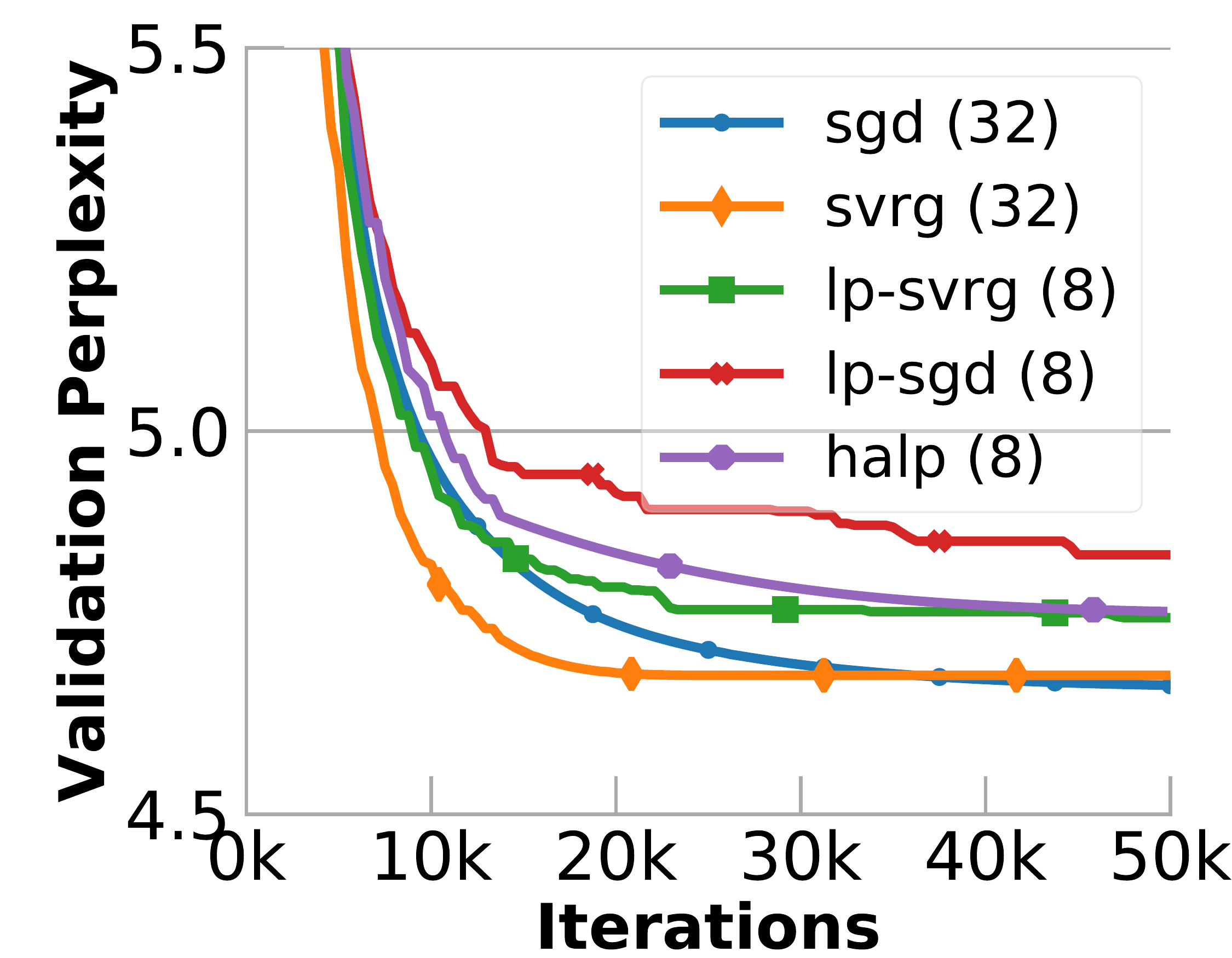}} \hfill
\end{tabular}
\caption{Training loss and validation perplexity on LSTM for character level language modeling with TinyShakespeare dataset and for ResNet for image recognition with CIFAR10 dataset. Training loss for is smoothed for visualization purposes. The CIFAR10 accuracy is monotonic because we report the best value up to each specific number of iterations, which is standard for reporting validation accuracy. 
} 
\label{fig:dl_results}
\end{figure*}

\section{Evaluation}
\label{sec:evaluation}
In this section, we empirically evaluate both LP-SVRG and \sysname{} on standard training 
applications. Our goal here is to validate (1) that these new
low-precision algorithms can lead to high-accuracy solutions and (2) that the increased 
throughput of these low-precision algorithms can lower the required end-to-end training 
time. In \Cref{sec:eval_nn} we validate
that both LP-SVRG and HALP are capable of achieving high-accuracy solutions (when 
compared to low-precision SGD) on both a convolutional neural network (CNN) and a recurrent neural networks (RNN) neural network. 
In \Cref{sec:eval_log_reg} we validate that on multi-class logistic regression,
HALP can lead to substantially higher accuracy solutions than LP-SVRG while
executing each epoch up to $\numb{4} \times$ faster 
than full-precision SVRG running on a commodity CPU.

\subsection{Deep Learning Results}
\label{sec:eval_nn}

For deep learning, we show that (1) LP-SVRG can exhibit better training losses than low-precision SGD, and (2) HALP can have even better training losses than LP-SVRG.

\begin{figure}
\centering     \hfill
\subfigure[MNIST.]{\label{fig:mnist_8_bit}\includegraphics[width=0.49\myfcwidth]{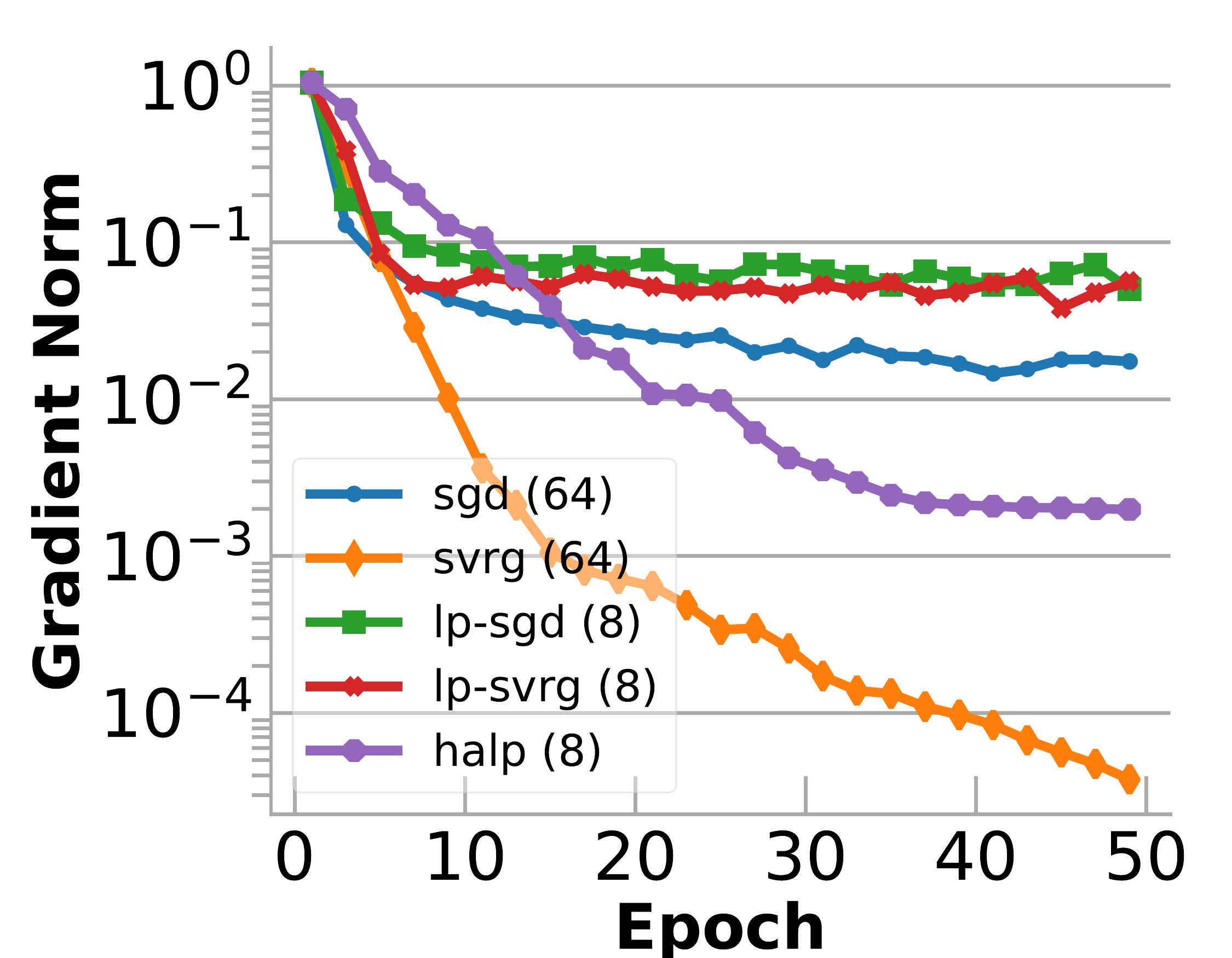}} \hfill
\subfigure[Synthetic dataset.]{\label{fig:classif_10000_8_bit}\includegraphics[width=0.49\myfcwidth]{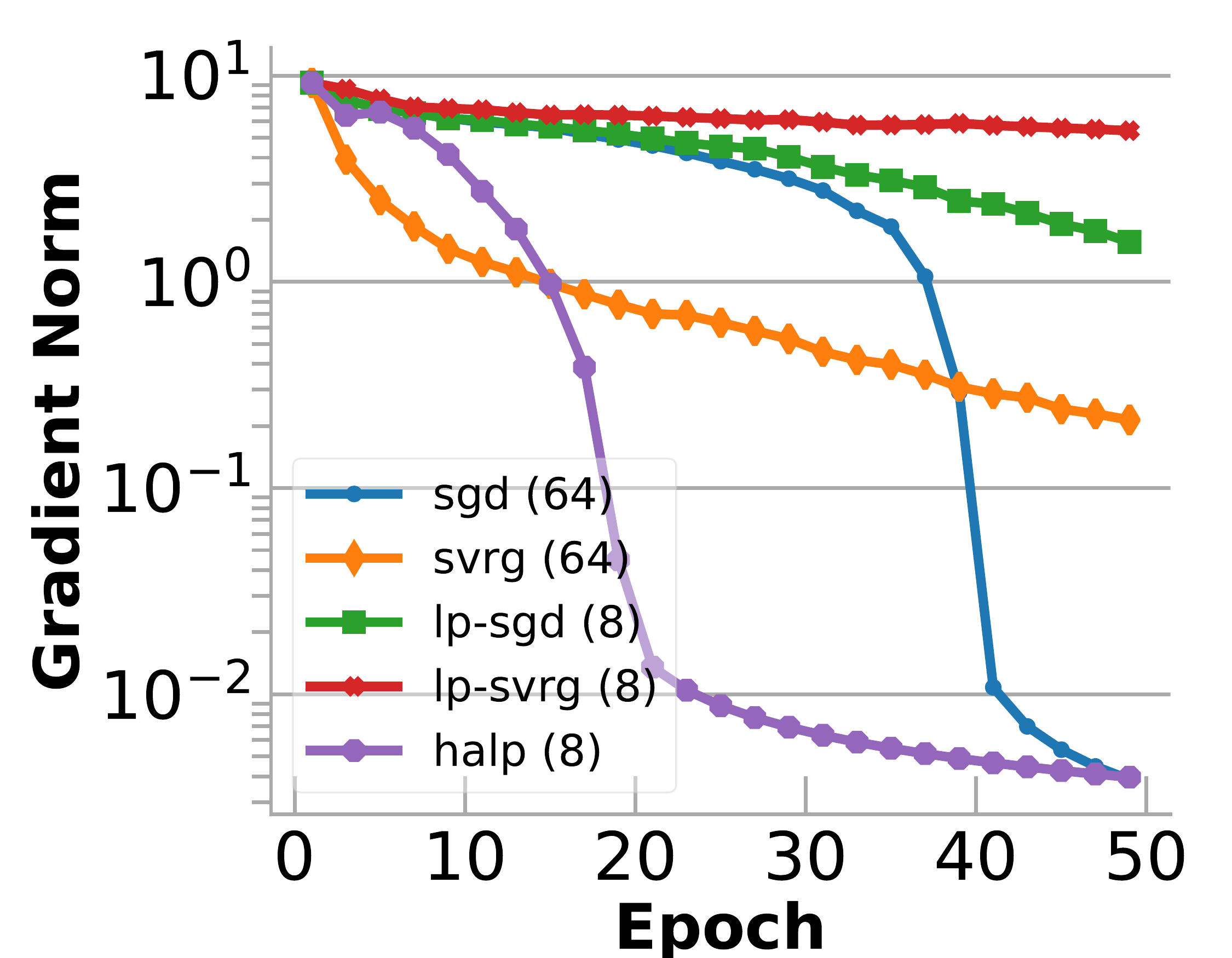}} \hfill
\vspace{-4mm}
\caption{Convergence of 8-bit low-precision algorithms on multi-class logistic regression training.}
\label{fig:halp_results}
\end{figure}

\textbf{Experimental Setup. } To demonstrate the effectiveness of LP-SVRG and HALP, 
we empirically evaluate them on both a CNN and 
a RNN. To run these experiments, we extended the 
TensorQuant~\cite{loroch2017tensorquant} toolbox to simulate the quantization 
operations from Algorithm~\ref{algLPSVRG} and Algorithm~\ref{algHALP}. In more detail, our deep learning training results run the computation at full-precision, but simulate the quantization operations during the model update.
We conduct the CNN experiment with a ResNet 
~\cite{he2016deep} architecture 
with $16, 32,$ and $64$ channels of $3\times3$ kernels in the first, 
second, and third building blocks. We train this model on the CIFAR10~\cite{krizhevsky2014cifar} image classification dataset. 
For the RNN experiment, we evaluate a character-level language modeling task on 
the TinyShakespeare dataset~\cite{karpathy2015visualizing}. 
The RNN model is a two-layer LSTM, where each layer contains with 
128 hidden units. 
In our 8-bit experiments, we uniformly set learning rate to 0.5 for the CNN experiments and to 1.0 for the RNN experiments. We sweep scale $\delta$ in grid $\{0.001, 0.002, 0.005, 0.01, 0.02, 0.05 \}$ for both the CNN and RNN experiments.
For HALP runs, we sweep $\mu$ in the grid $\{1, 2, 5, 10, 20, 50\}$.
For each algorithm, we present the training loss curve and validation metric curve with the best values found from the grid search.

\textbf{CNN Discussion. }
\Cref{fig:train_loss_cnn,fig:val_acc_cnn} show that 8-bit HALP comes close to 
the training loss of full-precision SVRG, while significantly outperforming LP-SVRG and LP-SGD.
In terms of validation accuracy, \Cref{fig:val_acc_cnn} 
shows that 8-bit LP-SVRG produces a model with improved validation 
accuracies of 0.8\% when compared to LP-SGD, and more importantly that 8-bit HALP 
reaches validation metrics that 
match the quality of full-precision SVRG. 

\textbf{LSTM Discussion. } 
In \Cref{fig:train_loss_lstm,fig:val_acc_lstm} we show that 8-bit LP-SVRG outperforms
LP-SGD in training loss and validation perplexity on the LSTM 
model. 
We show that 8-bit HALP is able to achieve a training loss that is better than LP-SVRG and closely matches the results from full-precision SVRG. 
We also observe that the validation perplexity from 8-bit HALP is worse than the one from 8-bit SVRG; this phenomenon (a lower training loss not guaranteeing better generalization) is often observed in deep learning results.

\subsection{Multi-Class Logistic Regression Results}
\label{sec:eval_log_reg}

We validate that HALP strictly outperforms LP-SVRG and LP-SGD on two logistic
regression training applications while executing each iteration up to $4\times$ 
faster than full-precision (64-bit) SVRG.

\textbf{Experimental Setup. } To test the effectiveness of the HALP algorithm, we compared 8-bit HALP to 
8-bit LP-SGD, 8-bit LP-SVRG,
64-bit SVRG, and 64-bit SGD implementations on (10 class) logistic regression
classification tasks. We implemented 8 and 16-bit versions of each algorithm in C++ 
using AVX2 intrinsics. 
We present results on two datasets, MNIST with 784 features and 60,000 samples
and a synthetic dataset with 10,000 features and 7,500 samples.
The synthetic dataset was generated
using scikit-learn's \texttt{make_classification} dataset generator.
In \Cref{sec:extended_log_reg} we present the complete experimental 
details.

\textbf{Statistical Discussion. } \Cref{fig:halp_results} shows that HALP strictly outperforms both LP-SGD and
LP-SVRG. Interestingly, \Cref{fig:mnist_8_bit} shows that HALP outperforms all 
algorithms except full-precision SVRG on the MNIST dataset. This is expected because, due to its use of bit centering, HALP has lower-magnitude quantization noise than both LP-SGD and LP-SVRG. Amazingly, \Cref{fig:classif_10000_8_bit} shows
that sometimes HALP is capable of outperforming all learning algorithms including full-precision SVRG.

\textbf{Performance Discussion. } \Cref{table:perf_numbers} shows that HALP outperforms SGD by up to
$\numb{3} \times$ and SVRG by up to $\numb{4} \times$ while remaining within 25\% of LP-SGD per epoch.
On MNIST the performance gains of HALP are less pronounced as this dataset is purely
compute bound---meaning we do not experience the memory bandwidth benefits of low-precision.
On the larger synthetic dataset the application becomes more memory bound, and
we correspondingly notice a larger performance improvement in the low-precision 
applications. 
To a certain degree these performance results are limited by the design of current
CPU architectures (see \Cref{sec:extended_log_reg}).

\begin{table}
  \begin{scriptsize}
  \begin{sc}
  \begin{center}
    \setlength{\tabcolsep}{20pt}
    \begin{tabular}{@{}rrr@{}}
      \toprule
      &                 MNIST               &Synth. 10,000 \\
      \midrule
      Best              &$0.16 \, \textrm{s}$              &$0.19 \, \textrm{s}$\\
      \midrule
      SGD               &$2.10 \times$              &$4.08 \times$\\
      SVRG              &$11.57 \times$             &$20.58 \times$\\
      SVRG+             &$2.70 \times$              &$5.30 \times$\\
      LP-SGD            &$\mathbf{1.00 \times}$     &$\mathbf{1.00 \times}$\\
      LP-SVRG           &$1.14 \times$              &$1.22 \times$\\
      HALP              &$1.13 \times$              &$1.24 \times$\\
      \bottomrule
    \end{tabular}
  \end{center}
  \end{sc}
  \end{scriptsize}
      \caption{Runtime (in seconds) per epoch of the best performing 
      optimization algorithm and relative runtime for all optimization algorithms.
      Logistic regression is run over classification datasets with 10 classes 
      and 10,000 samples. The low-precision
      algorithms (LP) are all run with 8-bit precision. SVRG is the original 
      code released by \citet{johnson2013accelerating} and SVRG+ is our implementation 
      that uses AVX2 intrinsics.}
    \label{table:perf_numbers}
\end{table}

\section{Conclusion}

In this paper we presented HALP, a new SGD variant that is able to theoretically
converge at a linear rate while using fewer bits. HALP leverages SVRG to reduce noise 
from gradient variance, and introduces \emph{bit centering} 
to reduce noise from quantization. 
To validate the effectiveness of SVRG in low-precision computation, we show 
that both HALP and low-precision SVRG converge to high-accuracy 
solutions on LSTM, CNN, and multi-class logistic regression applications 
up to $\numb{4} \times$ faster than full-precision SVRG.

\bibliographystyle{plainnat}
\bibliography{references}

\newpage
\onecolumn

\appendix

\section{Extended Evaluation}
\label{sec:extended_evlaution}

\subsection{Conditioning Evaluation}

Our results in Theorem~\ref{thmHALP} exposed a relationship between the condition number and the performance of low-precision training.
These results suggested that as the condition number is increased, there is a threshold (the threshold at which the conditions of the theorem hold) above which the performance of HALP is no longer guaranteed, and may become dramatically worse.
Here, we validate this intuition empirically.

In order to study this effect, we generated a series of linear regression problems with different condition numbers.
Each problem had as matrix of training examples $X \in \R^{64 \times 1000}$, and was generated from a singular value decomposition $X = U S V$ such that $U$ and $V$ were random orthogonal matrices, and the singular values $S$ were chosen such that for the linear regression problem, $\mu = 1$ and $L = \kappa$, for some desired condition number $\kappa$.
We then ran SVRG, 16-bit HALP, and 8-bit HALP on each dataset, using a fixed epoch length $T = 1000$, running for $K = 50$ epochs, and using grid search over a wide range of step sizes $\alpha$ (ranging from $10^{-2}$ to $10^{-10}$) and HALP parameters $\mu$ (ranging from $0.5$ to $2000$), and choosing the best parameter setting for each algorithm.
Figure~\ref{fig:halpconditioning} presents the gradient norm of the algorithms after $K = 50$ epochs.
Notice that as the condition number increases, all the algorithms perform worse (i.e. converge slower); this is a well-known effect caused by the fact that poorly conditioned problems are harder to solve.
A more interesting effect is the fact that for both 8-bit and 16-bit HALP, there is a sudden threshold where the performance of the algorithm suddenly degrades, and above which the performance is worse than SVRG.
This validates our theory, which predicts that such a thing will happen, and places a bound on the minimum number of bits that will be needed for solving problems of a particular condition number.

\begin{figure}
  \centering
  \includegraphics[width=0.4\linewidth]{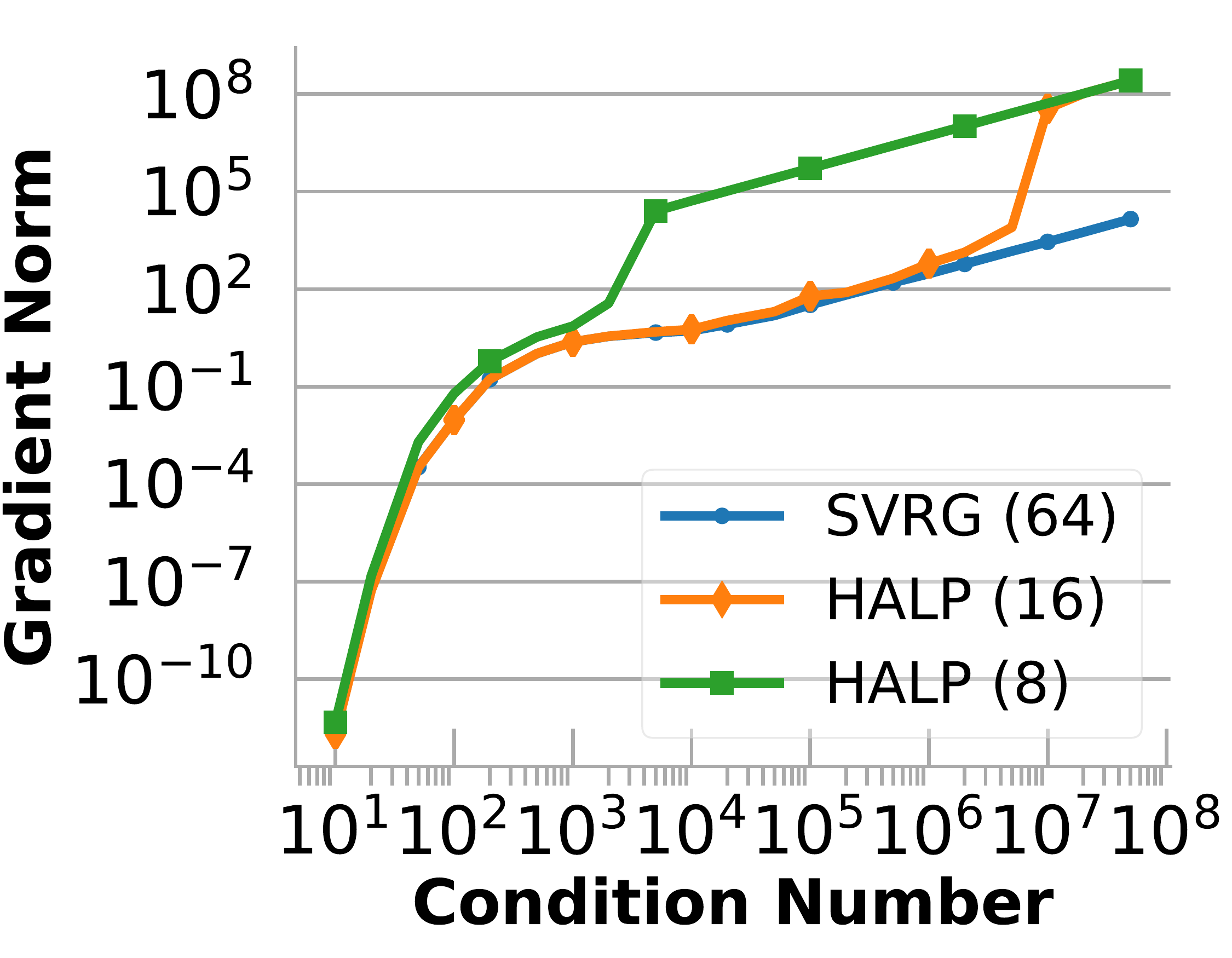}
  \caption{Additional linear regression experiment on synthetic dataset with an SVD(U,S,V) where
  U and V are random orthogonal matrices and the singular values S are such that the condition number of the resulting problem is $\kappa$. The gradient norm is measured after 50 epochs.}
  \label{fig:halpconditioning}
\end{figure}

\subsection{Extended Neural Network Evaluation}
\label{sec:extended_nn}

\paragraph{16-bit Experiments.} In \Cref{fig:dl_results_16} we show 16-bit 
experiments on the LSTM and CNN neural networks presented in \Cref{sec:eval_nn}. 
We use the same experiment protocol as with \Cref{sec:eval_nn}, except the grid for $\delta$ is
\[
  \delta \in \{0.00001, 0.00002, 0.00005, 0.0001, 0.0002, 0.0005\}.
\]
Our results for the most part follow the trends from the 8-bit results 
from \Cref{sec:eval_nn}. Still, we highlight here that 16-bit LP-SVRG
benefits from variance reduction in this low-precision training. 
As such, LP-SVRG achieves a training loss that matches the 
performance of 32-bit full-precision SVRG. We show that 16-bit 
LP-SVRG also consistently outperforms 16-bit SGD in training loss.

\begin{figure*}
\centering
\begin{tabular}{c@{\hspace{0.3em}} c@{\hspace{0.3em}} c@{\hspace{0.3em}} c}
  \subfigure[Training loss on ResNet.]{\label{fig:train_loss_cnn_16}\includegraphics[width=0.24\linewidth]{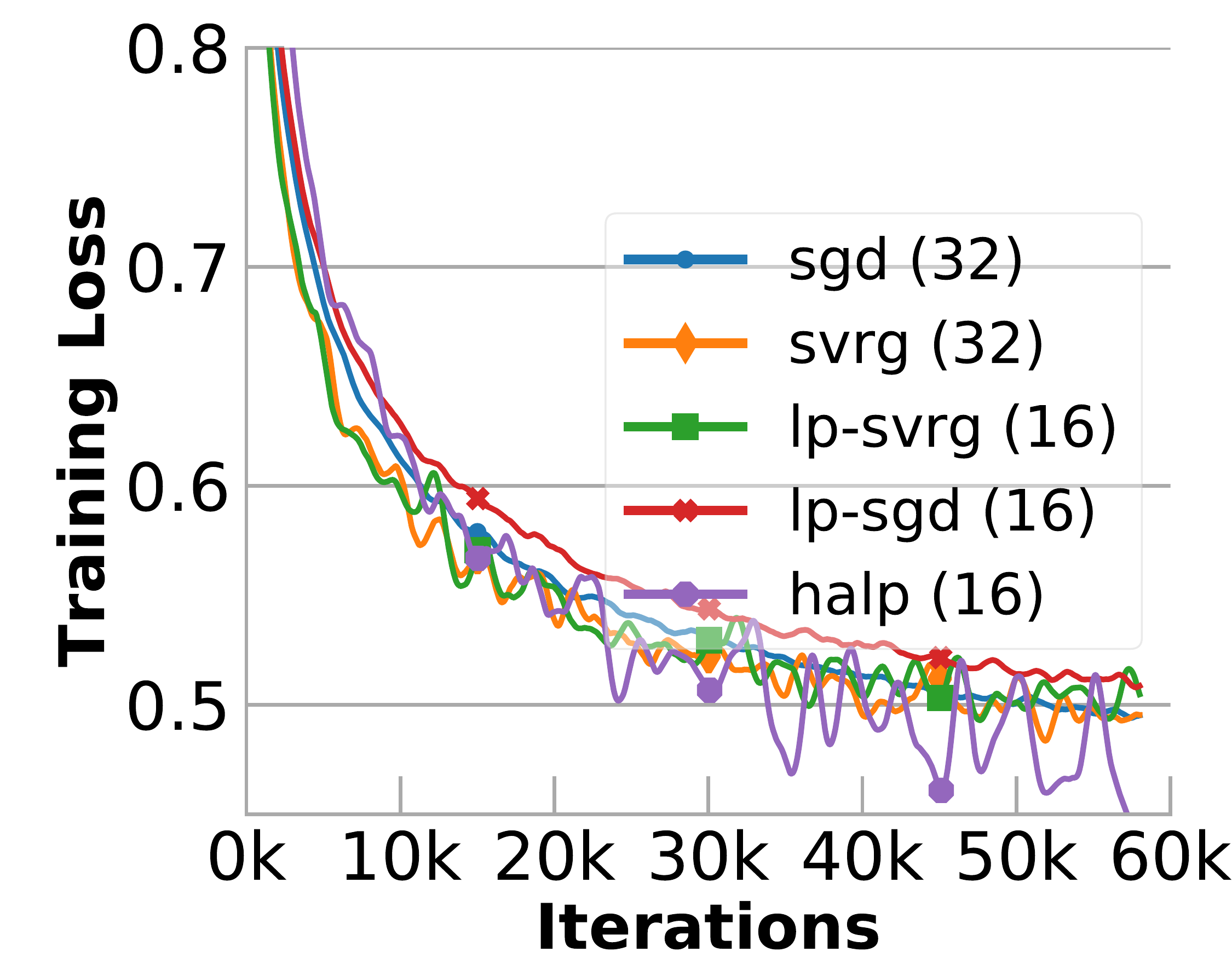}} \hfill
  \subfigure[Val. Accuracy on ResNet.]{\label{fig:val_acc_cnn_16}\includegraphics[width=0.24\linewidth]{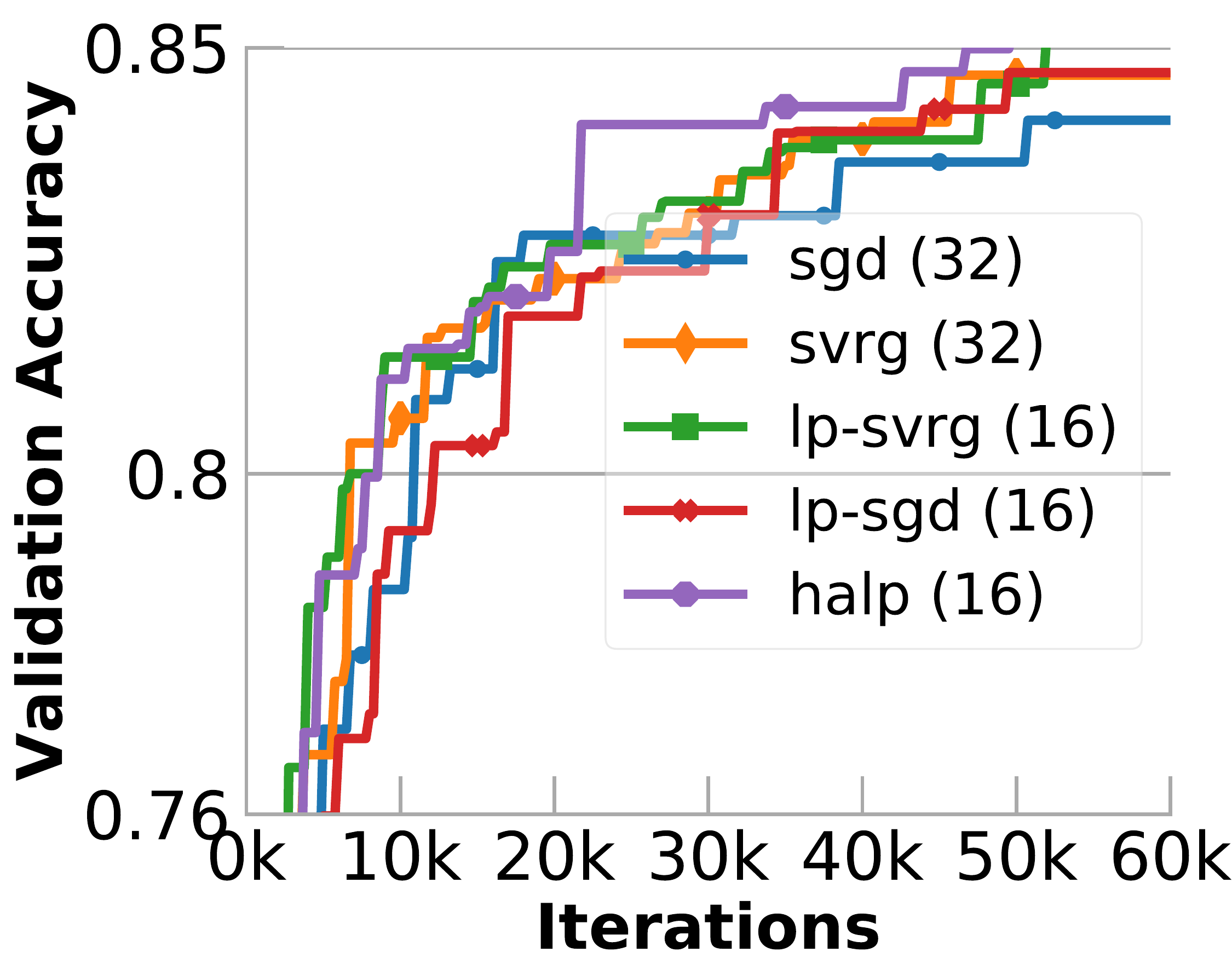}} \hfill
  \subfigure[Training loss on LSTM.]{\label{fig:train_loss_lstm_16}\includegraphics[width=0.24\linewidth]{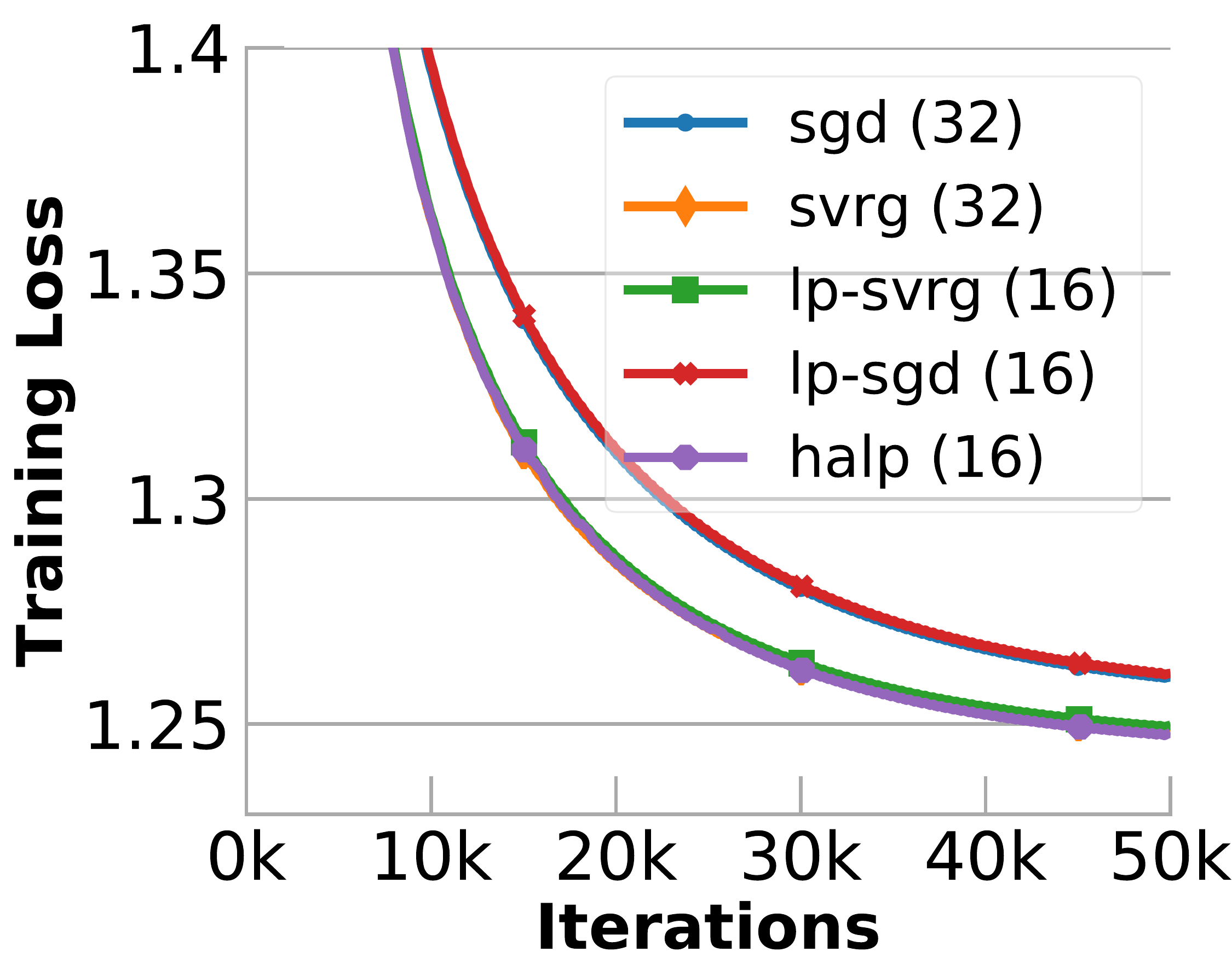}} \hfill
  \subfigure[Val. Accuracy on LSTM.]{\label{fig:val_acc_lstm_16}\includegraphics[width=0.24\linewidth]{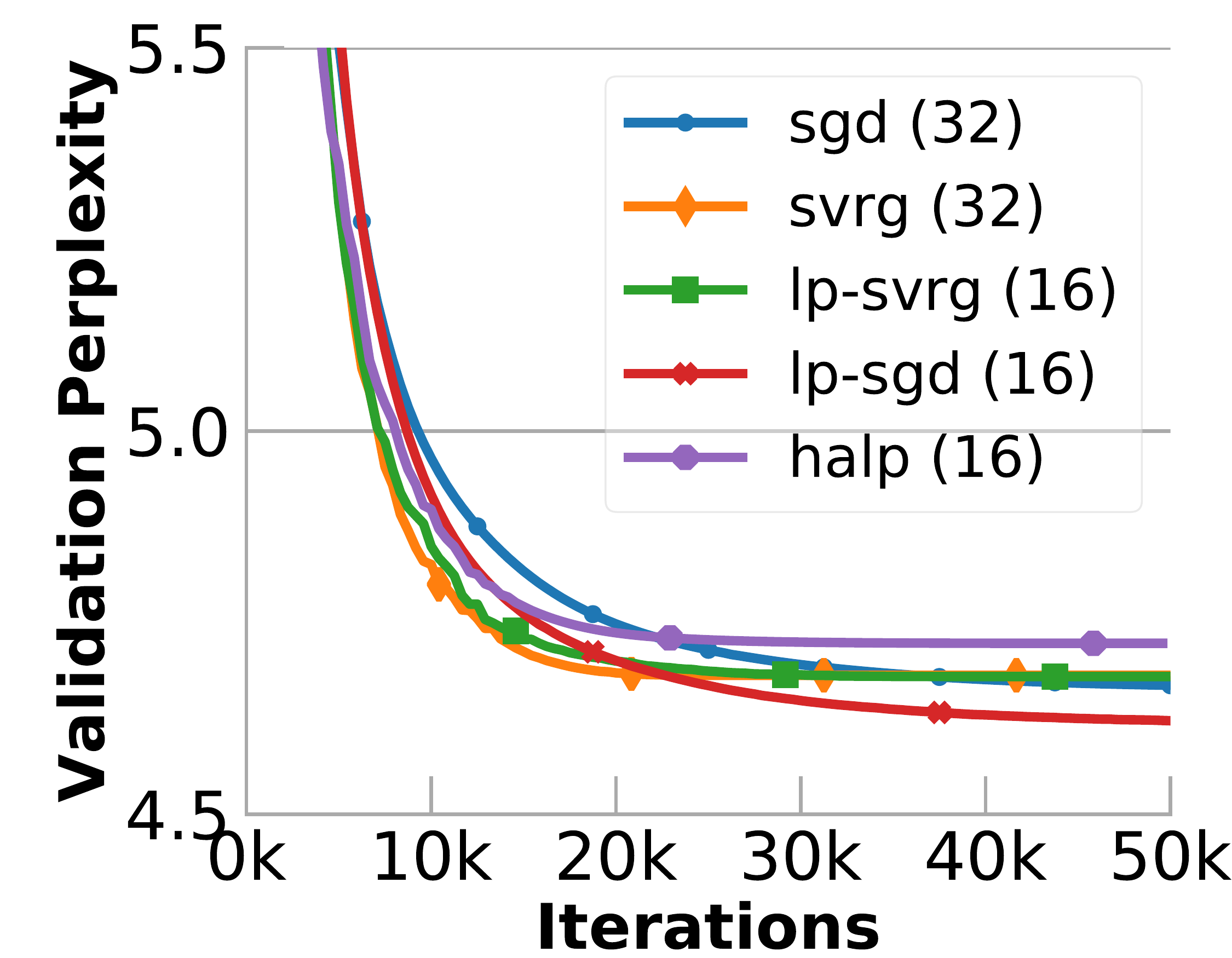}} \hfill
\end{tabular}
\caption{Training loss and validation perplexity on ResNet for image recognition with CIFAR10 dataset and for LSTM for character level language modeling with TinyShakespeare dataset. Training loss for LSTM is smoothed for visualization purposes. The CIFAR10 accuracy is the monotonic as we report the best value up to each specific number of iterations.} 
\label{fig:dl_results_16}
\end{figure*}

\subsection{Extended Multi-Class Logistic Regression Evaluation}
\label{sec:extended_log_reg}

\begin{table}
  \small
  \begin{center}
    \setlength{\tabcolsep}{20pt}
    \begin{tabular}{@{}rrrrrr@{}}
      \toprule
      hyperparameter    &SGD (64)      &SVRG (64) &LP-SGD (8) &LP-SVRG (8) &HALP (8)  \\
      \midrule
      $\alpha$          &1e-4           &1e-2        &1e-4       &1e-2      &4.5e-2\\
      L2 Reg.           &1e-4          &1e-4      &1e-4       &1e-4        &1e-4\\
      SVRG Internal     &-             &2         &-          &2           &2\\
      $\mu$             &-             &-         &-          &-           &2.5\\
      $w\_scale$         &-             &-        &2e-3       &2e-3        &-\\
      \bottomrule
    \end{tabular}
  \end{center}
      \caption{Exact hyperparameter settings for presented MNIST results.}
      \label{table:mnist_hyper}
\end{table}

\begin{table}
  \small
  \begin{center}
    \setlength{\tabcolsep}{20pt}
    \begin{tabular}{@{}rrrrrr@{}}
      \toprule
      hyperparameter    &SGD (64)      &SVRG (64) &LP-SGD (8) &LP-SVRG (8) &HALP (8)  \\
      \midrule
      $\alpha$          &7.5e-5        &1e-5      &7.5e-5     &7.5e-5      &7.5e-4\\
      L2 Reg.           &1e-4          &1e-4      &1e-4       &1e-4        &1e-4\\
      SVRG Internal     &-             &2         &-          &2           &2\\
      $\mu$             &-             &-         &-          &-           &256\\
      $w\_scale$         &-             &-         &1e-3       &1e-3        &-\\
      \bottomrule
    \end{tabular}
  \end{center}
      \caption{Exact hyperparameter settings for presented Synthetic 10,000-feature dataset results.}
      \label{table:synth_hyper}
\end{table}

\paragraph{Experimental Setup. } The synthetic dataset was generated with \texttt{n_informative=n_features} and otherwise default parameters (besides those already mentioned like \texttt{n_features}). To select hyperparameters we ran a 100-point
grid search for each training algorithm and present the  
configurations that minimized the gradient norm after training for 50 epochs.
\Cref{table:synth_hyper,table:mnist_hyper} shows the exact hyperparameters 
used for the runs present. For all SVRG-based algorithms (SVRG, LP-SVRG, and HALP) we compute
the full gradient every two epochs as done in \citet{johnson2013accelerating}. Note our definition of epoch for the experimental results means one full pass over the dataset, whereas in the algorithmic definitions, it referred to one outer loop iteration of SVRG or HALP. 
To test statistical convergence we measure the gradient norm every other epoch
because it goes to zero as the algorithm converges, and for strongly convex functions it is at most a constant factor away from other common metrics such as the objective gap and the distance to the optimum. 
To ensure robustness we run each algorithm with five different random seeds and 
present the average 
of the results. We ran all experiments on a single machine with a total of 56
cores on four Intel Xeon E7-4850 v3 CPUs and 1 TB of RAM. For the SVRG-based 
algorithms (SVRG, LP-SVRG, and HALP),
we parallelized the computation of the full gradient as this is embarrassingly 
parallel (and does not affect the statistical results).
We run these algorithms with 56 threads.  For end-to-end performance, we measure 
the wallclock time for each epoch and report the average of 5 epoch timings.

\paragraph{Performance Discussion. } The stochastic quantization is expensive, even when implemented in AVX2 using the \texttt{XORSHIFT}~\cite{marsaglia2003xorshift} pseudorandom number generator. Generally, we found we also need more instructions to operate on fixed-point types, limiting the throughput benefits we can achieve from using low precision. A common example arises with multiplication, where we need to perform an additional shift following the multiplication of two fixed-point types to maintain the same bit width. Moreover, current architectures do not support efficiently operating on types less than 8 bits, limiting our expected performance improvement to 8-bit and 16-bit fixed-point types. Still, these results are exciting and we hope 
will influence the design of future architectures
and hardware accelerators.

\paragraph{Clipping Discussion. } In the body of the paper, we explained how the effect observed in Figure~\ref{fig:classif_10000_8_bit}, where 8-bit HALP outperforms plain SVRG, can be attributed to a gradient-clipping-like effect caused by saturating arithmetic.
Here, we provide more detailed evidence of this claim.

The hypothesis that we want to validate is that the improved performance of 8-bit HALP over SVRG is caused by the saturating arithmetic: the fact that the distance the iterates can move in a single epoch is bounded by the box of representable numbers, as illustrated in Figure~\ref{figCenterScale}.
Additionally, we want to rule out the possibility that the quantization itself is causing this improvement, or that the noise caused by quantizing to a particular scale is responsible.
To investigate this hypothesis, we ran two additional experiments.
The first (which we call `Scale') runs 16-bit HALP with the same scale (the same $\delta$) as 8-bit HALP would use at every iteration.
To do this, we adjusted the value of $\mu$ so that $\mu_8 (b^{8-1} - 1) = \mu_{16} (b^{16-1} - 1)$.
The second (which we call `Clip') runs 16-bit HALP with the same range of representable numbers (the same $\mu$) as 8-bit HALP.
If our hypothesis is correct, we would expect the second 16-bit HALP run, which has the same gradient-clipping-like effect as 8-bit HALP, to mimic its performance and outperform SVRG.
We would also expect the first 16-bit HALP run to be closer to SVRG's performance, since it has a much larger range of representable numbers and so will not exhibit the same gradient-clipping-like effect.

Figure~\ref{fig:additionalclipping} presents the results of our experiment.
Our hypothesis was validated, and in fact the 16-bit HALP runs are nearly indistinguishable from the trajectories we hypothesized they would take.
This strongly suggests that the gradient-clipping-like effect \emph{is indeed responsible} for the improved performance of 8-bit HALP over SVRG in Figure~\ref{fig:classif_10000_8_bit}.

\begin{figure}
  \centering
  \includegraphics[width=0.4\linewidth]{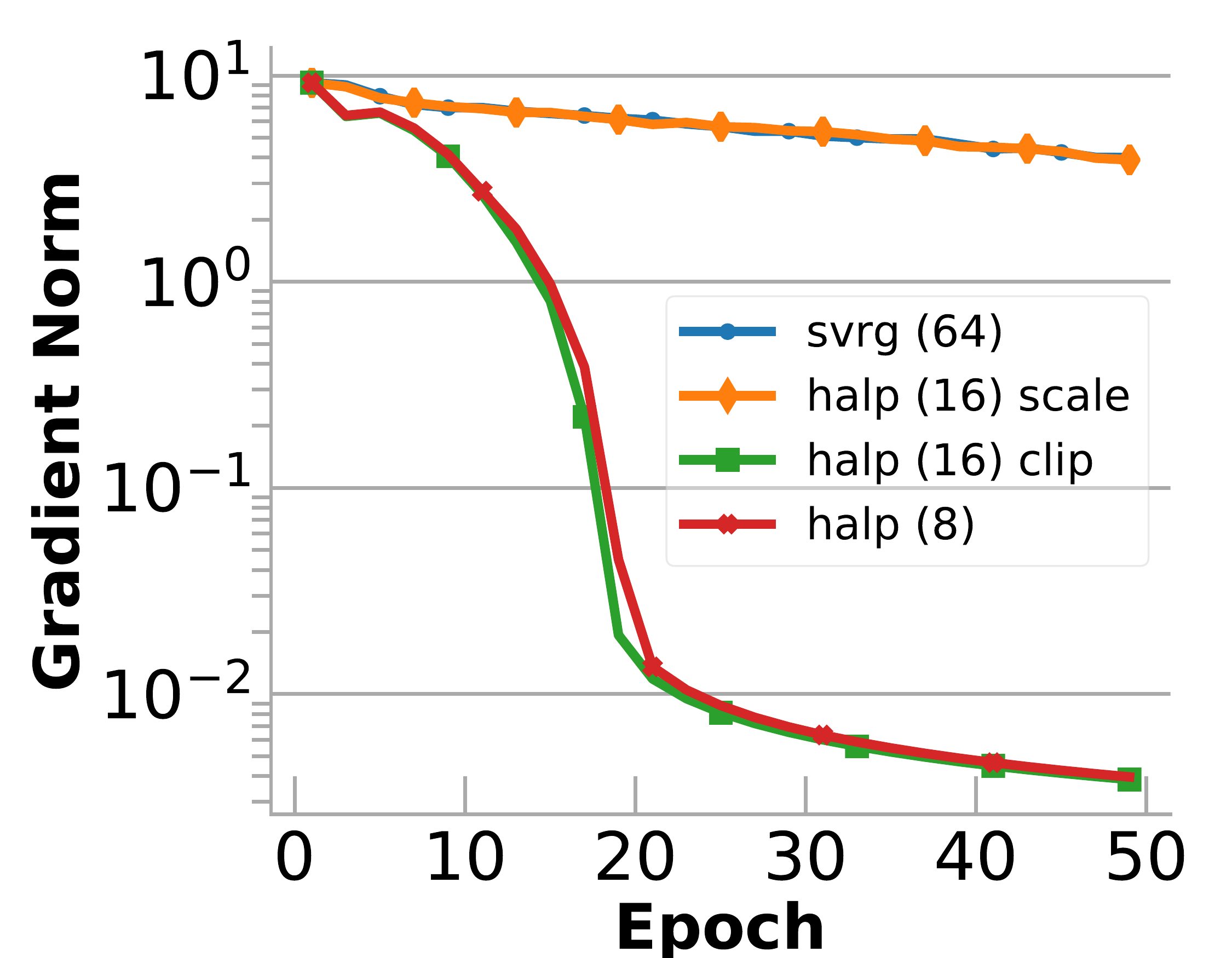}
  \caption{Additional experiments on synthetic dataset with 10,000 features. `Scale'
  refers to the 16-bit HALP run with an adjusted $mu$ value of 1.
  This run has the same quantization scale as the 8-bit HALP algorithm.
  `Clip' refers to a 16-bit HALP run with the same $mu$ value as the 8-bit HALP (256).
  This run has the same range of representable values as the 8-bit HALP algorithm, and so will clip the gradients at the same level.}
  \label{fig:additionalclipping}
\end{figure}

\section{Proofs}

Before we prove the main theorems presented in the paper, we will prove the following lemmas, which will be useful later.

\begin{lemma}
  \label{lemmaQuantization}
  Under the above conditions where we quantize using the low-precision representation $(\delta, b)$, for any $w$,
  \[
    \Exv{ \norm{ Q_{(\delta,b)}(w) - w^* }^2 }
    \le
    \norm{ w - w^* }^2
    +
    \frac{d \delta^2}{4}.
  \]
\end{lemma}
\begin{proof}[Proof of Lemma~\ref{lemmaQuantization}]
  First, observe that this entire inequality separates additively along dimensions.
  Therefore, it suffices to prove it just for the case of $d = 1$.

  To prove it for $d = 1$, we can consider two cases separately. First, if $w$ is within the range representable by $(\delta, b)$, then $\Exv{Q_{(\delta,b)}(w)} = w$. In this case,
  \begin{dmath*}
    \Exv{ \norm{ Q_{(\delta,b)}(w) - w^* }^2 }
    =
    \Exv{ ( Q_{(\delta,b)}(w) - w^* )^2 }
    =
    \Exv{ ( (Q_{(\delta,b)}(w) - w) - (w - w^*) )^2 }
    =
    \Exv{ 
      (Q_{(\delta,b)}(w) - w)^2
      -
      2 (Q_{(\delta,b)}(w) - w) (w - w^*)
      +
      (w - w^*)^2
    }
    =
    \Exv{ (Q_{(\delta,b)}(w) - w)^2 }
    -
    2 (w - w) (w - w^*)
    +
    (w - w^*)^2
    =
    (w - w^*)^2
    +
    \Exv{ (Q_{(\delta,b)}(w) - w)^2 }.
  \end{dmath*}
  Since $w$ is within the representable range, it will either be rounded up or down at random.
  Let $z$ be the rounded-down quantization of $w$.
  Then $Q_{(\delta,b)}(w)$ will round to $z + \delta$ (the rounded-up quantization of $w$) with probability $\frac{w - z}{\delta}$,
  and it will round to $z$ with probability $\frac{z + \delta - w}{\delta}$.
  This quantization is unbiased because
  \[
    \Exv{ Q_{(\delta,b)}(w) }
    =
    \frac{w - z}{\delta} (z + \delta)
    +
    \frac{z + \delta - w}{\delta} z
    =
    \frac{wz - z^2 + w \delta - z \delta}{\delta}
    +
    \frac{z^2 + z \delta - wz}{\delta}
    =
    w.
  \]
  Thus, its variance will be
  \begin{dmath*}
    \Exv{ (Q_{(\delta,b)}(w) - w)^2 }
    =
    \frac{w - z}{\delta} (z + \delta - w)^2
    +
    \frac{z + \delta - w}{\delta} (z - w)^2
    =
    (w - z) (z + \delta - w) \left(
      \frac{z + \delta - w}{\delta}
      +
      \frac{w - z}{\delta}
    \right)
    =
    (w - z) (z + \delta - w)
    \le
    \frac{\delta^2}{4}.
  \end{dmath*}
  It follows that in the $d = 1$ case, when $w$ is on the interior of the representable region, 
  \begin{dmath*}
    \Exv{ \norm{ Q_{(\delta,b)}(w) - w^* }^2 }
    \le
    (w - w^*)^2
    +
    \frac{\delta^2}{4}.
  \end{dmath*}
  In the other case, when $w$ is on the exterior of the representable region, the quantization function $Q_{(\delta,b)}$ just maps it to the nearest representable value.
  Since $w^*$ is in the interior of the representable region, this operation will make $w$ closer to $w^*$.
  Thus,
  \[
    \norm{ Q_{(\delta,b)}(w) - w^* }^2 \le \norm{ w - w^* }^2,
  \]
  and so it will certainly be the case that
  \begin{dmath*}
    \Exv{ \norm{ Q_{(\delta,b)}(w) - w^* }^2 }
    \le
    (w - w^*)^2
    +
    \frac{\delta^2}{4}.
  \end{dmath*}
  We have now proved this inequality for all values of $w$, when $d = 1$.
  The inequality now follows in full generality by summing up over dimensions.
\end{proof}

For completeness, we also re-state the proof of following lemma, which was presented as equation (8) in \citet{johnson2013accelerating}.

\begin{lemma}
  \label{lemmaSVRG8}
  Under the standard condition of Lipschitz continuity, if $i$ is sampled uniformly at random from $\{1, \dots, N\}$, then for any $w$,
  \[
    \Exv{ \norm{ \nabla f_i(w) - \nabla f_i(w^*) }^2 }
    \le
    2L \left( f(w) - f(w^*) \right).
  \]
\end{lemma}
\begin{proof}[Proof of Lemma~\ref{lemmaSVRG8}]
  For any $i$, define
  \[
    g_i(w) = f_i(w) - f_i(w^*) - (w - w^*)^T \nabla f_i(w^*).
  \]
  Clearly, if $i$ is sampled randomly as in the lemma statement, $\Exv{g_i(w)} = f(w)$.
  But also, $w^*$ must be the minimizer of $g_i$, so for any $w$
  \begin{dmath*}
    g_i(w^*)
    \le
    \min_{\eta} g_i(w - \eta \nabla g_i(w))
    \le
    \min_{\eta} \left( g_i(w) - \eta \norm{ \nabla g_i(w) }^2 + \frac{\eta^2 L}{2} \norm{\nabla g_i(w)}^2 \right)
    =
    g_i(w) - \frac{1}{2L} \norm{\nabla g_i(w)}^2.
  \end{dmath*}
  where the second inequality follows from the Lipschitz continuity property.
  Re-writing this in terms of $f_i$ and averaging over all the $i$ now proves the lemma statement.
\end{proof}

Now we are ready to prove Theorem~\ref{thmLPSVRG}.
Our proof of this theorem follows the structure of the proof of the original SVRG convergence result in \citet{johnson2013accelerating}.

\begin{proof}[Proof of Theorem~\ref{thmLPSVRG}]
  We start by looking at the expected distance-squared to the optimum.
  \begin{dmath*}
    \Exv{ \norm{ w_{k,t} - w^* }^2 }
    =
    \Exv{ \norm{ Q_{(\delta, b)}(u_{k,t}) - w^* }^2 }.
  \end{dmath*}
  By Lemma~\ref{lemmaQuantization}, this can be bounded from above by
  \begin{dmath*}
    \Exv{ \norm{ w_{k,t} - w^* }^2 }
    \le
    \Exv{ \norm{ u_{k,t} - w^* }^2 } + \frac{d \delta^2}{4}.
  \end{dmath*}
  Applying the recursive definition of $u_{k,t}$ from the algorithm statement produces
  \begin{dmath*}
    \Exv{ \norm{ w_{k,t} - w^* }^2 }
    \le
    \Exv{ \norm{ w_{k,t-1} - w^* - \alpha \left(\nabla f_i(w_{k,t-1}) - \nabla f_i(\tilde w_k) + \tilde g_k \right) }^2 } + \frac{d \delta^2}{4}
    =
    \Exv{ 
      \norm{ w_{k,t-1} - w^* }^2
      -
      2 \alpha (w_{k,t-1} - w^*)^T \left(\nabla f_i(w_{k,t-1}) - \nabla f_i(\tilde w_k) + \tilde g_k \right)
      +
      \alpha^2 \norm{ \nabla f_i(w_{k,t-1}) - \nabla f_i(\tilde w_k) + \tilde g_k }^2
    } 
    +
    \frac{d \delta^2}{4}
    =
    \Exv{ 
      \norm{ w_{k,t-1} - w^* }^2
      -
      2 \alpha (w_{k,t-1} - w^*)^T \nabla f(w_{k,t-1})
      +
      \alpha^2 \norm{ \nabla f_i(w_{k,t-1}) - \nabla f_i(\tilde w_k) + \tilde g_k }^2
    } 
    +
    \frac{d \delta^2}{4}
    \le
    \Exv{ 
      \norm{ w_{k,t-1} - w^* }^2
      -
      2 \alpha (f(w_{k,t-1}) - f(w^*))
      +
      \alpha^2 \norm{ \nabla f_i(w_{k,t-1}) - \nabla f_i(\tilde w_k) + \tilde g_k }^2
    } 
    +
    \frac{d \delta^2}{4},
  \end{dmath*}
  where this last inequality follows from convexity of the function $f$.
  This second-order term can be further bounded by
  \begin{dmath*}
    \Exv{ \norm{ \nabla f_i(w_{k,t-1}) - \nabla f_i(\tilde w_k) + \tilde g_k }^2 }
    =
    \Exv{ \norm{ \nabla f_i(w_{k,t-1}) - \nabla f_i(w^*) - \left( \nabla f_i(\tilde w_k) - \nabla f_i(w^*) - \tilde g_k \right) }^2 }
    \le
    \Exv{ 2 \norm{ \nabla f_i(w_{k,t-1}) - \nabla f_i(w^*) }^2 + 2 \norm{ \nabla f_i(\tilde w_k) - \nabla f_i(w^*) - \tilde g_k }^2 }
    =
    \Exv{ 2 \norm{ \nabla f_i(w_{k,t-1}) - \nabla f_i(w^*) }^2 + 2 \norm{ \nabla f_i(\tilde w_k) - \nabla f_i(w^*) - \Exv[j \sim \mathrm{Unif}(1,\ldots,N)]{ \nabla f_j(\tilde w_k) - \nabla f_j(w^*) } }^2 }
    \le
    \Exv{ 2 \norm{ \nabla f_i(w_{k,t-1}) - \nabla f_i(w^*) }^2 + 2 \norm{ \nabla f_i(\tilde w_k) - \nabla f_i(w^*) }^2 }
  \end{dmath*}
  where the first inequality holds because $\norm{x + y}^2 \le 2 \norm{x}^2 + 2 \norm{y}^2$ and the second holds because the variance is always upper bounded by the second moment.
  We can now apply Lemma~\ref{lemmaSVRG8} to this last expression, which produces
  \begin{dmath*}
    \Exv{ \norm{ \nabla f_i(w_{k,t-1}) - \nabla f_i(\tilde w_k) + \tilde g_k }^2 }
    \le
    \Exv{ 4 L ( f(w_{k,t-1}) - f(w^*) ) + 4L ( f(\tilde w_k) - f(w^*) ) }.
  \end{dmath*}
  Substituting this into the expression above,
  \begin{dmath*}
    \Exv{ \norm{ w_{k,t} - w^* }^2 }
    \le
    \Exv{ 
      \norm{ w_{k,t-1} - w^* }^2
      -
      2 \alpha (f(w_{k,t-1}) - f(w^*))
      +
      4 L \alpha^2 \left( ( f(w_{k,t-1}) - f(w^*) ) + ( f(\tilde w_k) - f(w^*) ) \right)
    }
    +
    \frac{d \delta^2}{4}
    =
    \Exv{\norm{ w_{k,t-1} - w^* }^2}
    -
    2 \alpha (1 - 2 L \alpha) \Exv{ f(w_{k,t-1}) - f(w^*) }
    +
    4 L \alpha^2 \Exv{ f(\tilde w_k) - f(w^*) }
    +
    \frac{d \delta^2}{4}.
  \end{dmath*}
  Summing this up across all $T$ iterations of an epoch produces
  \begin{dmath*}
    \sum_{t=1}^T \Exv{ \norm{ w_{k,t} - w^* }^2 }
    \le
    \sum_{t=1}^T \Exv{\norm{ w_{k,t-1} - w^* }^2}
    -
    2 \alpha (1 - 2 L \alpha) \sum_{t=1}^T \Exv{ f(w_{k,t-1}) - f(w^*) }
    +
    4 L \alpha^2 T \Exv{ f(\tilde w_k) - f(w^*) }
    +
    \frac{d \delta^2 T}{4}.
  \end{dmath*}
  Now, canceling the terms from the first two sums, and noticing that $w_{k,0} = \tilde w_k$,
  \begin{dmath*}
    \Exv{ \norm{ w_{k,T} - w^* }^2 }
    \le
    \Exv{ \norm{ \tilde w_k - w^* }^2}
    -
    2 \alpha (1 - 2 L \alpha) \sum_{t=1}^T \Exv{ f(w_{k,t-1}) - f(w^*) }
    +
    4 L \alpha^2 T \Exv{ f(\tilde w_k) - f(w^*) }
    +
    \frac{d \delta^2 T}{4}.
  \end{dmath*}
  If we use option II to assign the next outer iterate, then
  \[
    \Exv{ f(\tilde w_{k+1}) - f(w^*) }
    =
    \frac{1}{T} \sum_{t=1}^T \Exv{ f(w_{k,t-1}) - f(w^*) },
  \]
  and so
  \begin{dmath*}
    \Exv{ \norm{ w_{k,T} - w^* }^2 }
    \le
    \Exv{ \norm{ \tilde w_k - w^* }^2}
    -
    2 \alpha (1 - 2 L \alpha) T \Exv{ f(\tilde w_{k+1}) - f(w^*) }
    +
    4 L \alpha^2 T \Exv{ f(\tilde w_k) - f(w^*) }
    +
    \frac{d \delta^2 T}{4}.
  \end{dmath*}
  As a consequence of the strong convexity property,
  \[
    \frac{\mu}{2} \norm{ \tilde w_k - w^* }^2 \le f(\tilde w_k) - f(w^*),
  \]
  so
  \begin{dmath*}
    2 \alpha (1 - 2 L \alpha) T \Exv{ f(\tilde w_{k+1}) - f(w^*) }
    \le
    \Exv{ \norm{ w_{k,T} - w^* }^2 }
    +
    2 \alpha (1 - 2 L \alpha) T \Exv{ f(\tilde w_{k+1}) - f(w^*) }
    \le
    \Exv{ \norm{ \tilde w_k - w^* }^2}
    +
    4 L \alpha^2 T \Exv{ f(\tilde w_k) - f(w^*) }
    +
    \frac{d \delta^2 T}{4}
    \le
    \frac{2}{\mu} \Exv{ f(\tilde w_k) - f(w^*) }
    +
    4 L \alpha^2 T \Exv{ f(\tilde w_k) - f(w^*) }
    +
    \frac{d \delta^2 T}{4}
    =
    \left( \frac{2}{\mu} + 4 L \alpha^2 T \right) \Exv{ f(\tilde w_k) - f(w^*) }
    +
    \frac{d \delta^2 T}{4},
  \end{dmath*}
  and dividing to isolate the left side,
  \begin{dmath*}
    \Exv{ f(\tilde w_{k+1}) - f(w^*) }
    \le
    \frac{ \frac{2}{\mu} + 4 L \alpha^2 T }{2 \alpha (1 - 2 L \alpha) T } \Exv{ f(\tilde w_k) - f(w^*) }
    +
    \frac{d \delta^2 T}{8 \alpha (1 - 2 L \alpha) T }
    \le
    \left( \frac{1}{\alpha \mu (1 - 2 L \alpha) T} + \frac{2 L \alpha}{1 - 2 L \alpha} \right) \Exv{ f(\tilde w_k) - f(w^*) }
    +
    \frac{d \delta^2}{8 \alpha (1 - 2 L \alpha) }.
  \end{dmath*}
  This is the same as the analogous expression for SVRG, except for the additional term that is a function of $\delta$.
  Now, suppose we want to have an expected contraction factor of $\gamma$ each epoch.
  That is, we want
  \[
    \gamma = \frac{1}{\alpha \mu (1 - 2 L \alpha) T} + \frac{2 L \alpha}{1 - 2 L \alpha}.
  \]
  This is equivalent to having
  \[
    \alpha (1 - 2 L \alpha) \gamma = \frac{1}{\mu T} + 2 L \alpha^2,
  \]
  which can be further reduced to
  \[
    0 = \frac{1}{\mu T} - \alpha \gamma + 2 L (1 + \gamma) \alpha^2.
  \]
  This equation only has solutions when the discriminant is non-negative, that is, when
  \[
    0 \le \gamma^2 - 4 \cdot \frac{1}{\mu T} \cdot 2 L (1 + \gamma).
  \]
  The minimal value of $T$ for which this will be able to hold will be when it holds with equality, or when
  \[
    T = \frac{8 L (1 + \gamma)}{\mu \gamma^2} = \frac{8 \kappa (1 + \gamma)}{\gamma^2}.
  \]
  If we choose this $T$, then the solution to the above quadratic equation is, by the quadratic formula, to set $\alpha$ such that
  \[
    \alpha = \frac{\gamma}{2 \cdot 2 L (1 + \gamma)} = \frac{\gamma}{4 L (1 + \gamma)}.
  \]
  We can see that these are the settings of $T$ and $\alpha$ prescribed in the theorem statement.
  With these settings of $T$ and $\alpha$, we get
  \begin{dmath*}
    \Exv{ f(\tilde w_{k+1}) - f(w^*) }
    \le
    \gamma \Exv{ f(\tilde w_k) - f(w^*) }
    +
    \frac{d \delta^2}{8 \alpha (1 - 2 L \alpha) }
    =
    \gamma \Exv{ f(\tilde w_k) - f(w^*) }
    +
    \frac{d \delta^2}{8 \cdot \frac{\gamma}{4 L (1 + \gamma)} \cdot \left(1 - 2 L \cdot \frac{\gamma}{4 L (1 + \gamma)} \right) }
    =
    \gamma \Exv{ f(\tilde w_k) - f(w^*) }
    +
    \frac{d \delta^2}{8 \cdot \frac{\gamma}{4 L (1 + \gamma)} \cdot \frac{2 + \gamma}{2 (1 + \gamma)} }
    =
    \gamma \Exv{ f(\tilde w_k) - f(w^*) }
    +
    \frac{d \delta^2 L (1 + \gamma)^2}{ \gamma (2 + \gamma) }
    \le
    \gamma \Exv{ f(\tilde w_k) - f(w^*) }
    +
    \frac{2 d \delta^2 L}{ \gamma },
  \end{dmath*}
  where in the last line we use the fact that $1 + \gamma \le 2$ and $2 + \gamma \ge 2$.
  Now subtracting the fixed point of this expression from both sides,
  \begin{dmath*}
    \Exv{ f(\tilde w_{k+1}) - f(w^*) }
    -
    \frac{2 d \delta^2 L}{ \gamma (1 - \gamma) }
    \le
    \gamma \Exv{ f(\tilde w_k) - f(w^*) }
    +
    \frac{2 d \delta^2 L}{ \gamma }
    -
    \frac{2 d \delta^2 L}{ \gamma (1 - \gamma) }
    =
    \gamma \Exv{ f(\tilde w_k) - f(w^*) }
    +
    \frac{2 d \delta^2 L}{ \gamma } \left( 1 - \frac{1}{1 - \gamma} \right)
    =
    \gamma \left(
      \Exv{ f(\tilde w_k) - f(w^*) }
      -
      \frac{2 d \delta^2 L}{ \gamma (1 - \gamma) }
    \right).
  \end{dmath*}
  It follows by applying this statement recursively that
  \begin{dmath*}
    \Exv{ f(\tilde w_{K+1}) - f(w^*) }
    -
    \frac{2 d \delta^2 L}{ \gamma (1 - \gamma) }
    \le
    \gamma^K \left(
      f(\tilde w_1) - f(w^*)
      -
      \frac{2 d \delta^2 L}{ \gamma (1 - \gamma) }
    \right),
  \end{dmath*}
  or
  \[
    \Exv{ f(\tilde w_{K+1}) - f(w^*) }
    \le
    \gamma^K \left( f(\tilde w_1) - f(w^*) \right)
    +
    \frac{2 d \delta^2 L}{ \gamma (1 - \gamma) }.
  \]
  This is what we wanted to prove.
\end{proof}

\begin{proof}[Proof of Theorem~\ref{thmHALP}]
  The analysis of the inner loop of \sysname{} is identical to the analysis of LP-SVRG.
  By using the same argument as in the proof of Theorem~\ref{thmLPSVRG}, we can get that
  \begin{dmath*}
    \Exv{ f(\tilde w_{k+1}) - f(w^*) }
    \le
    \left( \frac{1}{\alpha \mu (1 - 2 L \alpha) T} + \frac{2 L \alpha}{1 - 2 L \alpha} \right) \Exv{ f(\tilde w_k) - f(w^*) }
    +
    \frac{d \delta^2}{8 \alpha (1 - 2 L \alpha) }.
  \end{dmath*}
  Unlike for LP-SVRG, for \sysname{}, the value of $\delta$ changes over time.
  Specifically, it is assigned to
  \[
    \delta = \frac{\norm{\tilde g_k}}{\mu (2^{b-1} - 1)}.
  \]
  As a result, we have
  \begin{dmath*}
    \Exv{ f(\tilde w_{k+1}) - f(w^*) }
    \le
    \left( \frac{1}{\alpha \mu (1 - 2 L \alpha) T} + \frac{2 L \alpha}{1 - 2 L \alpha} \right) \Exv{ f(\tilde w_k) - f(w^*) }
    +
    \frac{d \Exv{ \norm{\tilde g_k}^2 } }{8 \alpha \mu^2 (1 - 2 L \alpha) (2^{b-1} - 1)^2}.
  \end{dmath*}
  From Lemma~\ref{lemmaSVRG8}, we know that
  \begin{dmath*}
    \norm{\tilde g_k}^2
    =
    \norm{\nabla f(\tilde w_k) - \nabla f(w^*)}^2
    \le
    2 L \left( f(\tilde w_k) - f(w^*) \right).
  \end{dmath*}
  Thus,
  \begin{dmath*}
    \Exv{ f(\tilde w_{k+1}) - f(w^*) }
    \le
    \left( \frac{1}{\alpha \mu (1 - 2 L \alpha) T} + \frac{2 L \alpha}{1 - 2 L \alpha} \right) \Exv{ f(\tilde w_k) - f(w^*) }
    +
    \frac{2 L d \Exv{f(\tilde w_k) - f(w^*)} }{8 \alpha \mu^2 (1 - 2 L \alpha) (2^{b-1} - 1)^2}
    =
    \left( 
      \frac{1}{\alpha \mu (1 - 2 L \alpha) T} 
      +
      \frac{2 L \alpha}{1 - 2 L \alpha}
      +
      \frac{2 L d }{8 \alpha \mu^2 (1 - 2 L \alpha) (2^{b-1} - 1)^2}
    \right)
    \Exv{ f(\tilde w_k) - f(w^*) }
    =
    \left( 
      \frac{1}{\alpha \mu (1 - 2 L \alpha)} 
      \left(
        \frac{1}{T}
        +
        \frac{2 \kappa d }{8 (2^{b-1} - 1)^2}
      \right)
      +
      \frac{2 L \alpha}{1 - 2 L \alpha}
    \right)
    \Exv{ f(\tilde w_k) - f(w^*) }.
  \end{dmath*}
  Now, if we define $\hat T$ such that
  \[
    \frac{1}{\hat T}
    = 
    \frac{1}{T}
    +
    \frac{2 \kappa d }{8 (2^{b-1} - 1)^2},
  \]
  then this expression reduces to
  \begin{dmath*}
    \Exv{ f(\tilde w_{k+1}) - f(w^*) }
    \le
    \left( 
      \frac{1}{\alpha \mu (1 - 2 L \alpha) \hat T}
      +
      \frac{2 L \alpha}{1 - 2 L \alpha}
    \right)
    \Exv{ f(\tilde w_k) - f(w^*) }.
  \end{dmath*}
  Next, suppose as before that we want to contract by a factor of $\gamma$ in expectation each step.
  That is, we need
  \[
    \gamma
    =
    \frac{1}{\alpha \mu (1 - 2 L \alpha) \hat T}
    +
    \frac{2 L \alpha}{1 - 2 L \alpha}.
  \]
  The analysis of this is identical to that in the proof of Theorem~\ref{thmLPSVRG}.
  By this same analysis, the minimal value of $\hat T$ for which this will be able to hold will be when
  \[
    \hat T = \frac{8 \kappa (1 + \gamma)}{\gamma^2}.
  \]
  and
  \[
    \alpha = \frac{\gamma}{4 L (1 + \gamma)}.
  \]
  In order for $\hat T$ to have this magnitude, we need
  \begin{dmath*}
    \frac{1}{T}
    =
    \frac{1}{\hat T}
    -
    \frac{2 \kappa d }{8 (2^{b-1} - 1)^2}
    =
    \frac{\gamma^2}{8 \kappa (1 + \gamma)}
    -
    \frac{2 \kappa d }{8 (2^{b-1} - 1)^2}
    =
    \frac{\gamma^2 (2^{b-1} - 1)^2}{8 \kappa (1 + \gamma) (2^{b-1} - 1)^2}
    -
    \frac{2 \kappa^2 d (1 + \gamma)}{8 \kappa (1 + \gamma) (2^{b-1} - 1)^2}
    =
    \frac{
      \gamma^2 (2^{b-1} - 1)^2
      -
      2 \kappa^2 d (1 + \gamma)
    }{
      8 \kappa (1 + \gamma) (2^{b-1} - 1)^2
    }.
  \end{dmath*}
  So, $T$ is
  \begin{dmath*}
    T
    =
    \frac{
      8 \kappa (1 + \gamma) (2^{b-1} - 1)^2
    }{
      \gamma^2 (2^{b-1} - 1)^2
      -
      2 \kappa^2 d (1 + \gamma)
    }
    =
    \frac{
      8 \kappa (1 + \gamma)
    }{
      \gamma^2
      -
      2 \kappa^2 d (1 + \gamma) (2^{b-1} - 1)^{-2}
    }.
  \end{dmath*}
  If we assign $\alpha$ and $T$ in this way, as they are given in the theorem statement, then
  \begin{dmath*}
    \Exv{ f(\tilde w_{k+1}) - f(w^*) }
    \le
    \gamma
    \Exv{ f(\tilde w_k) - f(w^*) },
  \end{dmath*}
  and the result now follows by induction.
\end{proof}

\end{document}